\documentclass{article}
\usepackage{graphicx} 

\usepackage{amsmath,amssymb,amsthm}
\usepackage{hyperref}
\usepackage{cleveref}
\usepackage[numbers]{natbib}
\usepackage[affil-it]{authblk}
\usepackage{CJKutf8}

\usepackage{tikz}
\usetikzlibrary{calc, patterns, shapes.geometric, arrows.meta, positioning, fit, backgrounds, shadows, matrix, quotes, arrows.meta}

\makeatletter
\tikzset{
  on layer/.code={
    \pgfonlayer{#1}\begingroup
    \aftergroup\endpgfonlayer
    \aftergroup\endgroup
  },
  box/.style={draw, rectangle, minimum size=3.5mm, inner sep=0pt, fill=white},
  arc/.style={
    {Circle[open, length=2mm, width=2mm, sep=-1mm]}-,
    on layer=background,
    every node/.style={on layer=main} 
  },
  double arc/.style={
    {Circle[open, length=2mm, width=2mm, sep=-1mm]}-{Circle[open, length=2mm, width=2mm, sep=-1mm]},
    on layer=background,
    every node/.style={on layer=main}
  }
}
\makeatother \usepackage{aliascnt}
\ifcsname theorem\endcsname
\else
    \newtheorem{theorem}{Theorem}
\fi
\ifcsname definition\endcsname
\else
    \newaliascnt{definition}{theorem}
    \newtheorem{definition}[definition]{Definition}
    \aliascntresetthe{definition}
\fi
\ifcsname lemma\endcsname
\else
    \newaliascnt{lemma}{theorem}
    \newtheorem{lemma}[lemma]{Lemma}
    \aliascntresetthe{lemma}
\fi
\ifcsname remark\endcsname
\else
    \newaliascnt{remark}{theorem}
    
    \aliascntresetthe{remark}
\fi

\title{Dialectics for Artificial Intelligence}
\author{Zhengmian Hu\\Adobe Research}
\date{December 2025}

\usetikzlibrary{backgrounds}

\begin{document}

\maketitle

\begin{abstract}
Can artificial intelligence discover, from raw experience and without human supervision, concepts that humans have discovered? 
One challenge is that human concepts themselves are \emph{fluid}: conceptual boundaries can shift, split, and merge as inquiry progresses (e.g., Pluto is no longer considered a planet). To make progress, we need a definition of ``concept'' that is not merely a dictionary label, but a structure that can be revised, compared, and aligned across agents.

We propose an algorithmic-information viewpoint that treats a concept as an \emph{information object} defined only through its structural relation to an agent's total experience. The core constraint is \emph{determination}: a set of parts forms a reversible consistency relation if any missing part is recoverable from the others (up to the standard logarithmic slack in Kolmogorov-style identities). This reversibility prevents ``concepts'' from floating free of experience and turns concept existence into a checkable structural claim. To judge whether a decomposition is \emph{natural}, we define \emph{excess information}, measuring the redundancy overhead introduced by splitting experience into multiple separately described parts.

On top of these definitions, we formulate \emph{dialectics} as an optimization dynamics: as new patches of information appear (or become contested), competing concepts bid to explain them via shorter conditional descriptions, driving systematic expansion, contraction, splitting, and merging. Finally, we formalize low-cost concept transmission and multi-agent alignment using small \emph{grounds}/seeds that allow another agent to reconstruct the same concept under a shared protocol, making communication a concrete compute-bits trade-off.
\end{abstract}
 \section{Introduction}

To build artificial intelligence, it is useful to look at human intelligence as a reference point: not as an ideal to imitate in every detail, but as evidence that certain mechanisms are possible. A human individual, and still more a human society, learns from finite and partial experience. Yet from that limited stream it gradually constructs a reusable system of concepts (objects such as ``water'', ``hand'', ``animal'', ``cause'', ``price'') that can be applied again and again to explain what is observed, to predict what will happen next, and to guide action.

We still lack satisfactory answers to several basic questions about this process.

Why do concepts arise at all, rather than remaining a list of disconnected episodes? Why, without any central coordinator, do independent people often end up with highly similar concepts that they can mutually recognize? And when experience accumulates and explanations are revised, why do concepts sometimes split, merge, and get renamed?

The first phenomenon is not a rare philosophical curiosity. It is a common empirical fact. Human groups on different continents, with different languages and customs, spontaneously form concepts corresponding to stable regularities in the world. Many cultures, for example, converge on something like ``water'': a kind of stuff that can be drunk, flows, freezes, boils, and so on. Even without a shared education system, people tend to aggregate a family of similar experiences into the same concept and then treat that concept as something that exists.

At the same time, it is worth resisting the opposite extreme: that concepts are simply ``out there'' with fixed boundaries. In the history of inquiry, high-level concepts repeatedly change their boundaries and even their identities. Ancient Greek astronomy, for a time, used different names for the ``morning star'' (Phosphorus / Eosphorus) and the ``evening star'' (Hesperus), and only later unified them as one object (nowadays Venus). Similar episodes of merging and boundary revision recur throughout science: electric versus magnetic, wave versus particle, and many others. Concepts are not static dictionary entries, they are reshaped by competing explanations and by the pressure to predict more. If concepts are both spontaneously generated and continually revised, then we must ask what makes a concept good (new, stable, and reusable) and what makes a concept bad (fragile, misleading, or eventually discarded).

A particularly revealing constraint comes from communication. Good concepts are not only useful for one mind, they are also easy to align across minds. In some everyday cases, alignment is almost costless.

Consider one minimal cooperative scene. Agent $A$ tells agent $B$ that the avocado is edible, intending to change $B$'s behavior (perhaps later $B$ will buy them or pick them without $A$'s guidance). For $A$ to succeed, $A$ must ensure that $B$'s concept for ``avocado'' matches $A$'s. Yet even if $B$ has never seen an avocado, $A$ typically does not need to transmit the full extension (the set of all avocados in the world). A few examples, typical features, or a short description usually suffice for $B$ to reconstruct a concept boundary close enough to $A$'s for practical purposes. This suggests that concept alignment is not the transmission of a huge set, but the transmission of a small seed that grows into the intended concept when placed in a shared world with shared phenomenon.

Cross-lingual transfer makes the same point more sharply. English and Chinese assign different symbols---\textit{horse} and ``\begin{CJK}{UTF8}{gbsn}马\end{CJK}''---to a shared concept, in many contexts a one-line dictionary mapping is enough for alignment. But if one tries to introduce ``horse'' to a culture with no horses in its experience, alignment becomes expensive: one must describe anatomy, behavior, and contrasts with other animals. The bits of information needed to communicate a concept drop dramatically when both sides already possess similar experience and have already abstracted it in similar ways.

These observations push the problem beneath words. We should not ask only how labels correspond across languages, but what kinds of structure can be reproduced between agents with minimal communication when their experiences overlap. What properties make a concept simultaneously (i) useful for prediction and action and (ii) cheap to transmit and re-identify?

This paper proposes a formal viewpoint based on algorithmic information theory. The central move is to treat a ``concept'' as an information object: a finite description that is meaningful only through its structural relation to the agent's total experience. Concretely, we model experience (or a chosen experience record) as a string $I$, and we model candidate concepts as strings that must fit together with $I$ under a strong reversibility requirement. We capture this requirement by a constraint we call a \emph{determination}: a collection of parts forms a determination if each part is effectively recoverable from the others (up to the standard logarithmic slack of Kolmogorov-style identities).

This reversibility condition does two jobs at once. First, it prevents ``concepts'' from floating free: a concept must be tied to the whole by explicit mutual recoverability, so the claim that a concept exists becomes a checkable structural statement rather than a mere naming convention. Second, it gives us a clean place to measure quality. Any lossless split of experience into parts can be judged by how much redundant description it introduces. We formalize this redundancy as \emph{excess information}: roughly, how many extra bits are paid when representing the whole by multiple separately described parts, instead of by the whole itself. Low excess corresponds to a ``natural'' articulation: the parts do not duplicate each other and do not introduce artifacts beyond what was already present.

With these pieces, concept formation, evolution, and alignment can be treated within one computational frame: optimization. We call the resulting optimization dynamics \emph{dialectics}. At any time we have a feasible determination-based decomposition of experience into concepts, as new experience arrives (or as attention shifts to a previously neglected patch), concepts compete to explain the contested information. The concept that can provide the shorter conditional description gains the right to absorb that patch, boundaries move, and over time concepts grow, shrink, split, or merge as a consequence of repeated local improvements to a shared description-length objective.

Finally, the same framework gives a precise reading of ``near-costless alignment''. If two agents share the same public reconstruction protocol, then to communicate a concept it can be enough to transmit a small \emph{ground} (a seed or anchor) that pins down the intended side of a decomposition. The receiver can then run the same dialectical growth procedure to reconstruct the concept's extension inside their own experience record. In this way, alignment becomes a concrete communication-computation trade: fewer transmitted bits require more reconstruction work, but both are governed by the same reversible constraints and the same code-length criterion.

\paragraph{Contributions.}
The paper develops this program in four steps:
\begin{enumerate}
\item We give a structural definition of concepts as parts in a reversible consistency constraint (\emph{determination}), turning ``concept existence'' into a testable information-theoretic condition.
\item We introduce a unified criterion for evaluating concept decompositions via \emph{excess information}, measuring the redundancy cost of splitting in description-length terms.
\item We formulate \emph{dialectics} as an optimization dynamics: concepts compete to explain new experience by shortening conditional descriptions, yielding an operational mechanism for concept growth and boundary revision.
\item We formalize low-cost concept transmission and multi-agent alignment using small \emph{grounds}/seeds that allow another agent to reconstruct (and therefore point to) the same concept under a shared protocol.
\end{enumerate}

The guiding idea throughout is simple: concepts should be understood not as fixed labels, but as compact, reproducible descriptions that remain tied to experience by reversible constraints, and whose boundaries are stabilized by the pressure to compress and predict.
 \section{Background on Kolmogorov Complexity}
\label{sec:kolmogorov-background}

Kolmogorov complexity formalizes the intuition that an object is \emph{simple} if it admits a short effective description and \emph{complex} if it does not. Formally, $K(x)$ is the length of the shortest program that outputs $x$. The notion is computational (descriptions are effective programs) and simultaneously a compression measure (it is an idealized minimal code length).  This section summarizes the definitions and basic properties we will use \citep{Kolmogorov1965,Solomonoff1964,LiVitanyi2019}.

\subsection{Notation and Conventions}

We work with finite binary strings $x \in \{0,1\}^\ast$. Let $|x|$ denote the bit-length of $x$, and $\log$ denote $\log_2$.
We fix a standard computable pairing function $\langle x,y\rangle \in \{0,1\}^\ast$ that is bijective and efficiently decodable, define $(x,y)$ as shorthand for $\langle x,y\rangle$.

Algorithmic-information identities often hold up to small additive terms. We write
\[
a \stackrel{+}{=} b
\quad\Longleftrightarrow\quad
a = b \pm O(\log (n+1)),
\]
where $n$ is any upper bound on the relevant complexities in the expression (e.g., $n = \max\{K(x),K(y)\}$). Similarly, $a \stackrel{+}{\le} b$ means $a \le b + O(\log (n+1))$. When a stronger $O(1)$ slack holds, we will state it explicitly.
The $O(\log(n+1))$ slack should be understood asymptotically, uniformly over a family of instances for which the relevant complexities grow. Concretely, when we use a bound such as $K(x) \le f(x) + O(\log K(x))$, we implicitly fix a (usually infinite) family of strings $\{x\}$ and require that there exists a family of effective description methods whose additional overhead, beyond $f(x)$, is bounded by $c_1\log K(x)+c_2$ for some constants $c_1, c_2$ and for all $x$ in the family.

\subsection{Prefix-Free Kolmogorov Complexity}

Let $U$ be a fixed \emph{universal prefix-free Turing machine} (a universal machine whose set of halting programs is prefix-free).
The \emph{(prefix-free) Kolmogorov complexity} of $x$ is
\[
K(x) \;:=\; \min\{\,|p| \;:\; U(p)=x\,\}.
\]
The \emph{conditional} complexity of $x$ given $y$ is
\[
K(x\mid y) \;:=\; \min\{\,|p| \;:\; U(p,y)=x\,\},
\]
where $y$ is provided to $U$ as an auxiliary input (``side information'').
We define the \emph{joint} complexity
\[
K(x,y) \;:=\; K(\langle x,y\rangle),\qquad
K(x,y\mid z) \;:=\; K(\langle x,y\rangle \mid z).
\]
Let $x^\ast$ denote a shortest (prefix-free) program for $x$, i.e., $|x^\ast|=K(x)$ and $U(x^\ast)=x$.

\paragraph{Invariance theorem (machine-independence).}
If $U$ and $U'$ are two universal prefix-free machines, then there exists a constant $c$ (depending only on $U,U'$) such that for all $x$,
\[
|K_U(x)-K_{U'}(x)| \le c,
\qquad
|K_U(x\mid y)-K_{U'}(x\mid y)| \le c.
\]
Thus, after fixing a reference universal machine, $K(\cdot)$ is well-defined up to an additive constant \citep{LiVitanyi2019}.

\subsection{Basic Bounds and Incompressibility}

A string can always be described by itself plus a constant-size decoder, hence
\[
K(x) \stackrel+\le |x|.
\]
A counting argument shows that most strings are incompressible:
for any $n$ and any $c\ge 1$, at least a $(1-2^{-c})$ fraction of strings $x\in\{0,1\}^n$ satisfy
\[
K(x) \ge n-c.
\]
Such strings have no description significantly shorter than listing all their bits \citep{LiVitanyi2019}.

\subsection{Uncomputability and Upper Semicomputability}

A central fact is that $K(\cdot)$ is not computable. If we could compute any $K(x)$ exactly, we could decide nontrivial halting properties by searching for minimal programs, contradicting the undecidability of the halting problem \citep{LiVitanyi2019}.

What \emph{is} possible is approximation from above:
$K(x)$ is \emph{upper semicomputable} (a.k.a.\ computably enumerable from above).
There exists an effective procedure that enumerates programs $p$ in increasing length, runs them dovetailed, and whenever it finds a program producing $x$ it updates an upper bound on $K(x)$. Hence we can always obtain \emph{constructive upper bounds} on complexity by exhibiting an explicit description/program, but we cannot in general certify optimality (i.e., prove matching lower bounds).

\subsection{Chain Rules and Symmetry of Information}

Kolmogorov complexity obeys ``chain rules'' analogous to Shannon information, with logarithmic slack \citep{Shannon1948,CoverThomas2006}:
\begin{alignat*}{2}
    K(x,y) &= K(x) + K(y\mid x^\ast)+O(1) &{}&= K(y) + K(x\mid y^\ast)+O(1)\\
           &\stackrel{+}{=} K(x) + K(y\mid x) &{}&\stackrel{+}{=} K(y) + K(x\mid y).
\end{alignat*}
A useful corollary is a form of \emph{symmetry of information}:
\begin{alignat*}{2}
    K(x) + K(y) - K(x,y) &= K(x) - K(x\mid y^\ast)+O(1) &{}&= K(y) - K(y\mid x^\ast)+O(1)\\
    &\stackrel{+}{=} K(x) - K(x\mid y) &{}&\stackrel{+}{=} K(y) - K(y\mid x).
\end{alignat*}
These identities formalize the idea that the ``overlap'' between $x$ and $y$ is (approximately) symmetric \citep{LiVitanyi2019}.

\subsection{Algorithmic Mutual Information and Conditional Mutual Information}

The \emph{algorithmic mutual information} between $x$ and $y$ is defined by
\[
I(x;y) \;:=\; K(x) + K(y) - K(x,y).
\]
Equivalently (by symmetry of information),
\[
I(x;y) \stackrel{+}{=} K(x) - K(x\mid y) \stackrel{+}{=} K(y) - K(y\mid x).
\]
Heuristically, $I(x;y)$ measures how many bits of description can be shared between $x$ and $y$ when they are described jointly. When $I(x;y)\approx 0$ (up to logarithmic terms), $x$ and $y$ are algorithmically close to independent.

The \emph{conditional mutual information} given $z$ is
\[
I(x;y\mid z) \;:=\; K(x\mid z) + K(y\mid z) - K(x,y\mid z).
\]
Conditional mutual information is nonnegative up to logarithmic slack:
\[
I(x;y\mid z) \stackrel{+}{\ge} 0.
\]

\paragraph{Data processing (algorithmic form).}
If $f$ is a total computable function, then processing cannot create algorithmic information:
\[
I(f(x);y \mid f) \stackrel{+}{\le} I(x;y \mid f),
\qquad
I(f(x);y\mid z, f) \stackrel{+}{\le} I(x;y\mid z, f).
\]
This is the algorithmic analogue of Shannon's data processing inequality \citep{LiVitanyi2019,Shannon1948,CoverThomas2006}.

\subsection{Connection to Probability, Compression, and MDL}

Although $K(x)$ is defined via programs, it is tightly connected to coding and probabilistic modeling.

\paragraph{Universal a priori probability and the coding theorem.}
Define the universal semimeasure
\[
m(x) \;:=\; \sum_{p:\,U(p)=x} 2^{-|p|}.
\]
Because the domain of $U$ is prefix-free, Kraft's inequality ensures $\sum_x m(x)\le 1$~\citep{CoverThomas2006}.
The \emph{coding theorem} states
\[
K(x) = -\log m(x)+O(1).
\]
Thus, short descriptions correspond to high universal probability mass, making $K(x)$ an (idealized) optimal code length \citep{Levin1974,LiVitanyi2019}.

\paragraph{Relating $K$ to any computable model.}
Let $P$ be a computable probability mass function over $\{0,1\}^\ast$ (or over $\{0,1\}^n$ for fixed $n$). Then for all $x$,
\[
K(x) \stackrel{+}{\le} -\log P(x) + K(P).
\]
Interpretation: to describe $x$, one may first describe the model $P$ and then encode $x$ using an optimal Shannon code for $P$ \citep{Shannon1948,CoverThomas2006}. This inequality is the algorithmic-information backbone of the Minimum Description Length (MDL) principle: choose the model that minimizes
\[
K(P) - \log P(x),
\]
or its prequential variant when $x$ is a sequence \citep{Rissanen1978,Grunwald2007,Dawid1984}.

\subsection{How We Work with Upper Bounds in Practice}

Because $K(\cdot)$ is uncomputable, most algorithmic-information arguments proceed by:
(i) fixing a reference description language (a universal machine),
(ii) constructing explicit programs/codes to obtain upper bounds on $K(\cdot)$ and $K(\cdot\mid\cdot)$,
and (iii) using general inequalities (chain rules, data processing, incompressibility) to relate these bounds.
In applications, concrete compressors, probabilistic models, or learning systems act as effective description methods, providing computable upper bounds that stand in for the ideal $K$. \section{Algorithmic Parity Structures (Determinations)}
\label{sec:determination}

Many settings can be viewed as a form of \emph{general information segmentation}: we take some information and represent it by several ``parts'' (or ``components'') without committing to any particular modality or geometry.
For example, a ``segmentation'' may mean selecting a subset from a collection, decomposing a vector into multiple summands, partitioning a 2D/3D signal into regions, carving out text spans inside a long document, or separating a mixed audio recording into sources.
What matters here is not the domain-specific meaning of a part, but a shared structural requirement: there is a short piece of \emph{glue} program that makes the parts fit together in a reversible, consistency-preserving way.

Concretely, the glue should support two operations.
First, given the glue together with \emph{all} parts, one can recover the whole joint information they represent (equivalently, the representation is lossless).
Second, given the glue and \emph{all but one} part, one can recover the missing part.
In the special case of bits, XOR parity is exactly such a glue: it is a constant-size rule that lets any missing bit be reconstructed from the others.

We formulate an algorithmic-information analogue of this pattern for arbitrary binary strings, measuring the intrinsic cost of the glue by its Kolmogorov description length.
We call the resulting object an \emph{algorithmic parity structure}, or simply a \emph{determination}.

\subsection{From XOR to Algorithmic Parity}

For bits, the $(n+1)$-th \emph{parity bit} $x_{n+1}=x_1\oplus\cdots\oplus x_n$ has the property that for every $i\in\{1,\dots,n+1\}$ the missing bit $x_i$ is determined by the remaining $n$:
\[
x_i \;=\; x_1\oplus\cdots\oplus x_{i-1}\oplus x_{i+1}\oplus\cdots\oplus x_{n+1}.
\]
Equivalently, there is a single very short program $p_{\oplus}$ such that for every $i$, given $i$ and all bits except $x_i$, the universal machine can recover $x_i$:
\[
U(p_{\oplus}, i, x_1,\dots,x_{i-1},x_{i+1},\dots,x_{n+1}) \;=\; x_i.
\]
An algorithmic parity structure is the direct generalization where the bits are replaced by arbitrary strings and the XOR law is replaced by an arbitrary computable relation, but we ask that a \emph{single} program works uniformly for all missing positions~\citep{LiVitanyi2019}.

\subsection{Definition}
\label{sec:determination-def}

Let $x_1,\dots,x_n$ be finite binary strings.  For each $i$, write $x_{-i}$ for the $(n-1)$-tuple obtained by deleting $x_i$:
\[
x_{-i}\;:=\;(x_1,\dots,x_{i-1},x_{i+1},\dots,x_n).
\]
In the following definition, we will also optionally allow an auxiliary side-information string $z$ (``context'').

\begin{definition}[Algorithmic parity structure / determination]
\label{def:determination}
Fix $n\ge 2$.
An $n$-tuple $x_{1:n}$ is called an \emph{algorithmic parity structure} (a.k.a.\ a \emph{determination}) \emph{relative to} context $z$
if each component is algorithmically recoverable from the remaining components:
\begin{equation}
\forall i\in[n],\qquad
K(x_i \mid x_{-i}, z)\;\stackrel{+}{=}\;0.
\label{eq:det-mutual-recovery}
\end{equation}

We refer to each component $x_i$ as a \emph{concept} (a ``piece'' of information) participating in this determination.
\end{definition}

A pair $(A,B)$ forms a determination iff
\[
K(A\mid B)\stackrel{+}{=}0
\quad\text{and}\quad
K(B\mid A)\stackrel{+}{=}0.
\]
In this case $A$ and $B$ are algorithmically equivalent (they encode the same information up to $O(\log)$ terms).

We visualize a determination by a small square, called an \emph{algorithmic parity node} (or \emph{determination node}).
Each incident edge is labeled by a binary string, representing a participating \emph{concept}.
For $n=2$, the node has two incident edges.

\begin{figure}[h]
\centering
\begin{tikzpicture}
    \node[box] (N) at (0,0) {};
    \draw (N.west) -- ++(-0.8, 0) node[left] {$A$};
    \draw (N.east) -- ++(0.8, 0) node[right] {$B$};
\end{tikzpicture}
\caption{A 2-way determination between concepts $(A,B)$: each concept is recoverable from the other. The square denotes an algorithmic parity node / determination node, each edge carries a binary string (a concept).}
\label{fig:determination-2}
\end{figure}
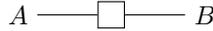

A triple $(A,B,C)$ forms a determination relative to $z$ iff knowing any two concepts (together with $z$)
allows reconstructing the third. This can be visualized as a 3-way relation.

\begin{figure}[h]
\centering
\begin{tikzpicture}
    \node[box] (N) at (0,0) {};
    \draw (N.north) -- ++(0, 0.6)   node[above] {$A$};
    \draw (N.south west) -- ++(-0.6, -0.6) node[below left] {$B$};
    \draw (N.south east) -- ++(0.6, -0.6)  node[below right] {$C$};
\end{tikzpicture}
\caption{A 3-way determination among concepts $(A,B,C)$: any two concepts determine the remaining one.}
\label{fig:determination}
\end{figure}
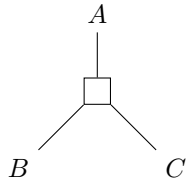

In case recoverability only holds given some shared side information $z$, we visualize such context by attaching a small arc labeled $z$ to the node, indicating that the same $z$ is provided to every reconstruction procedure associated with this node.

\begin{figure}[h]
\centering
\begin{tikzpicture}
    \node[box] (A) at (0,0) {};

    \draw[arc] (A.north) -- ++(0, 0.6)   node[above] {$z$};
    \draw (A.west) -- ++(-0.6, 0) node[left] {$A$};
    \draw (A.east) -- ++(0.6, 0) node[right] {$B$};
\end{tikzpicture}
\caption{A determination between $(A,B)$ given a shared context $z$ (attached to the determination node with an additional arc).}
\label{fig:determination-2-context}
\end{figure}
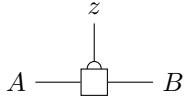

This arc notation reflects the fact that ``conditioning on $z$'' can be \emph{materialized} by duplicating $z$ as extra components. For example, the statement that $(A,B)$ forms a determination relative to $z$,
\[
K(A\mid B,z)\stackrel{+}{=}0
\quad\text{and}\quad
K(B\mid A,z)\stackrel{+}{=}0,
\]
is equivalent to the statement that there is unconditional determination on the 4-tuple $(A,B,z,z)$, since each of $A$ and $B$ can be recovered given the other and $z$, while each copy of $z$ is trivially recoverable from the remaining copy. More generally, an $n$-way determination relative to $z$ can be viewed as an unconditional determination on $(x_1,\dots,x_n, z, z)$, so that the context is effectively ``copied'' into two incident edges.

\subsection{Determinations as Constraints: One-Variable Completion}
\label{sec:det-one-var-opt}

In this paper we will repeatedly treat an algorithmic parity structure not only as a property to check, but as a \emph{constraint} that needs to be \emph{enforced} in optimization. A basic instance is \emph{one-variable completion}: we fix all but one component and optimize the remaining component's description length subject to the determination constraints.

Concretely, fix an integer $n\ge 2$, and fixed strings $x_{1:n}=(x_1,\dots,x_n)$. We introduce a single decision variable $x_{n+1}\in\{0,1\}^\ast$ and consider the constrained optimization problem
\begin{equation}
\min_{x_{n+1}\in\{0,1\}^\ast}\; K(x_{n+1})
\quad\text{s.t.}\quad
\begin{cases}
K(x_{n+1}\mid x_{1:n})\stackrel{+}{=}0,\\[2pt]
\forall i\in[n],\; K(x_i\mid x_{-i}, x_{n+1})\stackrel{+}{=}0.
\end{cases}
\label{eq:one-var-completion-opt}
\end{equation}
The constraint system in \eqref{eq:one-var-completion-opt} is exactly the requirement that
$(x_1,\dots,x_n,x_{n+1})$ forms a determination. We denote an optimal completion (i.e., a minimizer) of \eqref{eq:one-var-completion-opt} by $x_{n+1}^{\star}$, and the optimal value by $K(x_{n+1}^{\star})$.

\paragraph{Trivial feasibility.}
The feasible set of \eqref{eq:one-var-completion-opt} is never empty. For example, setting $x_{n+1}:=(x_1,\dots,x_n)$ satisfies the constraints, hence
\[
K(x_{n+1}^{\star})\stackrel{+}{\le}K(x_{1:n}).
\]
The role of the objective in \eqref{eq:one-var-completion-opt} is precisely to select, among all feasible completions, one whose description length $K(x_{n+1})$ is as small as possible.

\paragraph{Trivial lower bound.}
For any feasible choice of $x_{n+1}$ in~\eqref{eq:one-var-completion-opt}, each original component $x_i$ must be recoverable from the remaining components and $x_{n+1}$. By the standard subadditivity of conditional complexity,
\[
K(x_i \mid x_{-i})
\stackrel{+}{\le}
K\bigl(x_i \mid x_{-i}, x_{n+1}\bigr)
+ K\bigl(x_{n+1} \mid x_{-i}\bigr)
\stackrel{+}{=}
K\bigl(x_{n+1} \mid x_{-i}\bigr)
\stackrel{+}{\leq}
K\bigl(x_{n+1}\bigr),
\]
where the second step uses the feasibility constraint
$K(x_i \mid x_{-i}, x_{n+1})\stackrel{+}{=}0$. Taking the maximum over $i$ yields the necessary condition
\begin{equation}
K(x_{n+1})
\stackrel{+}{\ge}
\max_{i\in[n]} K(x_i \mid x_{-i}).
\label{eq:one-var-lb}
\end{equation}

This provides an information-theoretic lower bound on the objective in~\eqref{eq:one-var-completion-opt}.

\paragraph{Tight upper bound.}

The remaining question is whether the lower bound $\max_{i\in[n]} K(x_i \mid x_{-i})$ is effectively attainable. We now state that, it is.

\begin{theorem}[One-variable completion cost]
\label{prop:one-var-completion}
For any fixed $n\ge 2$ and strings $x_{1:n}$, the optimal value
of the constrained problem~\eqref{eq:one-var-completion-opt} is $\max_{i\in[n]} K(x_i \mid x_{-i})$, up to logarithmic slack. Equivalently, there exists a string $x_{n+1}$ such that
\begin{align}
K(x_{n+1}) &\;\stackrel{+}{\le} \max_{i\in[n]} K(x_i \mid x_{-i}),
\label{eq:one-var-ub}\\[3pt]
K(x_{n+1}\mid x_{1:n}) &\stackrel{+}{=} 0,
\qquad
\forall i\in[n],\;
K(x_i\mid x_{-i}, x_{n+1})\stackrel{+}{=}0.
\label{eq:one-var-feasibility}
\end{align}
\end{theorem}

In words, the minimal additional description length required to ``close'' $x_{1:n}$ into a determination is exactly the largest conditional complexity $K(x_i\mid x_{-i})$ among the components.

\paragraph{Intuition: ``a minimal diff'' that closes the parity node.}
It is helpful to first inspect the simplest nontrivial case $n=2$.
Fix two strings $(A,B)$. The one-variable completion problem asks for the shortest
$x_3$ such that $(A,B,x_3)$ is a determination, i.e.
\[
K(x_3\mid A,B)\stackrel{+}{=}0,\qquad
K(A\mid B,x_3)\stackrel{+}{=}0,\qquad
K(B\mid A,x_3)\stackrel{+}{=}0.
\]
In words, $x_3$ is a ``difference'' that is computable from the pair $(A,B)$
and makes $A$ and $B$ mutually recoverable given that difference. Thus, \Cref{prop:one-var-completion} can be viewed as the $n$-way generalization:
we are constructing a single short ``difference'' string that can complete \emph{any} missing position.

\paragraph{A sharp savings over a naive ``store all recovery programs'' encoding.}
A naive way to enforce determination on $(x_1,\dots,x_n,x_{n+1})$ is to store,
for each $i\in[n]$, a separate program $p_i$ witnessing $K(x_i\mid x_{-i})\le l$.
This would cost about $\sum_i |p_i|\le nl$ bits (ignoring logs), i.e.\ linear in $n$.

By contrast, \Cref{prop:one-var-completion} shows that a \emph{single} string of length
$l$ (up to logarithmic slack) suffices for \emph{all} $i$ simultaneously.

\paragraph{Contextual variant.}
Everything above relativizes to shared side information $z$.
Replacing each conditional complexity $K(\cdot\mid\cdot)$ by $K(\cdot\mid\cdot,z)$,
the same argument yields the existence of a completion $x_{n+1}$ with
\[
K(x_{n+1})\;\stackrel{+}{\le}\;\max_{i\in[n]}K(x_i\mid x_{-i},z),
\]
such that $(x_1,\dots,x_n,x_{n+1})$ forms a determination relative to $z$.

\paragraph{Connection to Muchnik-style multi-conditional descriptions.}
At first glance, \Cref{prop:one-var-completion} seems to involve ``many targets'': for each $i$ we require
$K(x_i\mid x_{-i},x_{n+1})\stackrel{+}{=}0$.
However, this distinction is largely superficial, because for every $i$ we have (up to logarithmic slack)
\begin{equation}
K(x_i\mid x_{-i})\;\stackrel{+}{=}\;K(x_{1:n}\mid x_{-i}).
\label{eq:xi_vs_joint_given_view}
\end{equation}

Thus, the completion string $x_{n+1}$ can be viewed as a \emph{single} auxiliary description that,
when paired with \emph{any} ``side information'' string $x_{-i}$, allows reconstructing the same
underlying object $x_{1:n}$ (and hence also the missing coordinate).
In this sense, \Cref{prop:one-var-completion} fits the template of Muchnik-style multi-conditional description:
a single short string simultaneously supports reconstruction under multiple conditions, with optimal length
governed by the worst-case conditional complexity $\max_i K(x_{1:n}\mid x_{-i})$~\citep{Muchnik2002,LiVitanyi2019}.

\paragraph{Connection to information distance in multiples.}
V{\'i}t{\'a}nyi's information distance for multiples characterizes the optimal length of a \emph{single} auxiliary program that can translate \emph{any} ``view'' $X_i$ into \emph{any} other $X_j$~\citep{Vitanyi2011,LiVitanyi2019}:
\[
E_{\max}(X_{1:m})
\;:=\;
\min\bigl\{|p|:\ \forall i,j,\ U(p,X_i,j)=X_j\bigr\}
\;\stackrel{+}{=}\;
\max_{i\in[m]} K(X_{1:m}\mid X_i).
\]
In our setting, take the $n$ views to be the $(n\!-\!1)$-projections $X_i:=x_{-i}$.
Although the list of all views $X_{1:n}=(x_{-1},\dots,x_{-n})$ is much longer as a raw string, it is
algorithmically equivalent to the joint tuple $x_{1:n}$ (up to $O(\log n)$ bookkeeping for indexing/formatting),
since $x_{1:n}$ trivially generates all $x_{-i}$, and conversely all $x_{-i}$ together determine $x_{1:n}$.
Moreover, for any fixed $i$ and any $j\neq i$, the task of reconstructing another view $X_j=x_{-j}$ from $X_i=x_{-i}$
is (up to $O(1)$) equivalent to reconstructing the missing component $x_i$ from $x_{-i}$.
Therefore our completion string $x_{n+1}$ can be interpreted as precisely such a ``universal translator'':
together with any view $X_i=x_{-i}$, it lets us recover the missing piece $x_i$ and hence the entire joint object,
which in turn yields any other view $X_j$.
Under this identification, \Cref{prop:one-var-completion} matches the multiple-information-distance principle,
and can be viewed as a specialization of V{\'i}t{\'a}nyi's characterization to the view family
$X_i=x_{-i}$~\citep{Vitanyi2011,Bennett1998,LiVitanyi2019}.

\paragraph{Beyond one-variable completion.}
One-variable completion is the cleanest case: the optimal glue cost is exactly a max of conditionals.
If instead we keep \emph{two} components free (e.g.\ fix $x_{1:n-1}$ and optimize a \emph{pair}
$(x_n,x_{n+1})$ jointly under determination constraints), the situation becomes substantially more intricate:
the two variables can trade off information in nontrivial ways.
Such interactions between concepts allow genuinely ``structural'' (and not merely ``diff-like'') determinations to appear during the optimization process.

\subsection{Concept Expansion}
\label{sec:concept-expansion}

Determinations are most useful when they are composed into \emph{networks} rather than studied in isolation. In a network, each determination node expresses a local ``any $n-1$ determine the last'' consistency constraint, and the overall graph encodes how multiple concepts cohere globally. This subsection introduces a basic \emph{local rewrite} of such networks, which we will later interpret as a primitive form of ``concept growth'' (and, in applications, as a mechanism for reallocating contested information).

The smallest nontrivial network is a chain of two 3-way determinations that share a single concept. Concretely, consider two diagrams in \Cref{fig:pivot-move}.

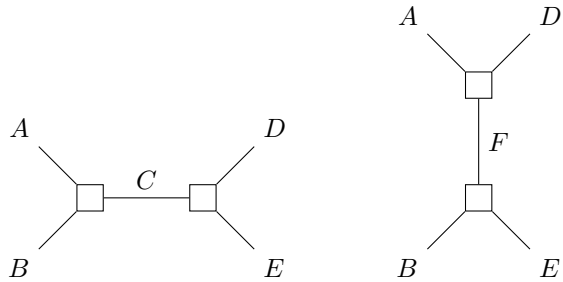
\begin{figure}[h]
\centering
\begin{tikzpicture}
    \node[box] (L) at (0,-2.5) {};
    \node[box] (R) at (1.5,-2.5) {};

    \draw (L.east) -- (R.west) node[midway, above] {$C$};
   
    \draw (L.north west) -- ++(-0.5, 0.5) node[above left] {$A$};
    \draw (L.south west) -- ++(-0.5, -0.5) node[below left] {$B$};
   
    \draw (R.north east) -- ++(0.5, 0.5) node[above right] {$D$};
    \draw (R.south east) -- ++(0.5, -0.5) node[below right] {$E$};
\end{tikzpicture}
\hspace{1.0cm}
\begin{tikzpicture}
    \node[box] (TOP) at (0, 0) {};
    \node[box] (BOT) at (0, -1.5) {};

    \draw (TOP.south) -- (BOT.north) node[midway, right] {$F$};

    \draw (TOP.north west) -- ++(-0.5, 0.5) node[above left] {$A$};
    \draw (TOP.north east) -- ++(0.5, 0.5) node[above right] {$D$};

    \draw (BOT.south west) -- ++(-0.5, -0.5) node[below left] {$B$};
    \draw (BOT.south east) -- ++(0.5, -0.5) node[below right] {$E$};
\end{tikzpicture}
\caption{A \emph{pivot move} for two adjacent 3-way determinations. Left: $(A,B,D,E)$ form a network with two determination nodes that share $C$ as a bridge. Right: same strings $(A,B,D,E)$ form a different network with (usually a different string) $F$ as a bridge.}
\label{fig:pivot-move}
\end{figure}

The right-hand picture is \emph{not} implied by the left-hand one in general. It is an additional property that may or may not hold.

When both diagrams hold, we say that the two-node chain \emph{admits a pivot}, and (as an \emph{interpretive reading}) that \emph{inside $A$, the concept $E$ expands beyond $C$ to become $F$}. By symmetry of determination nodes, the very same configuration can also be read as \emph{inside $B$, the concept $D$ expands beyond $C$ to become $F$}, and likewise for the other symmetric readings. We will freely use whichever phrasing best matches the surrounding interpretation.

\paragraph{Why we call it ``expansion''.}
Suppose the left-hand network in \Cref{fig:expansion-nested} holds with $(A,X,P_1,P_2)$ as external concepts and $I$ as bridge \footnote{Note that $I$ here denotes an information object (a binary string). When $I(\cdot;\cdot)$ or $I(\cdot;\cdot|\cdot)$  appears as a function, it always denotes (algorithmic) mutual information, and we never treat an information object $I$ as a function, so the two uses are unambiguous and do not conflict.}. 
The same network can be redrawn (purely graphically) in \Cref{fig:expansion-nested} as ``$A$ decomposes into $(X,I)$, and then $I$ decomposes into $(P_1,P_2)$''.
If in addition a pivot exists producing a new bridge $P_1'$, the pivoted network can be redrawn in \Cref{fig:expansion-absorbing} as
``$A$ decomposes into $(P_1',P_2)$, and then $P_1'$ decomposes into $(X,P_1)$''.
In this view, $P_1'$ has \emph{absorbed} $X$ together with the old $P_1$.
This is exactly the picture that ``inside $A$, $P_1$ expands beyond $I$ to become $P_1'$''.

\begin{figure}[h]
\centering
\begin{tikzpicture}
    \node[box] (L) at (0,-2.5) {};
    \node[box] (R) at (1.5,-2.5) {};

    \draw (L.east) -- (R.west) node[midway, above] {$I$};
   
    \draw (L.north west) -- ++(-0.5, 0.5) node[above left] {$A$};
    \draw (L.south west) -- ++(-0.5, -0.5) node[below left] {$X$};
   
    \draw (R.north east) -- ++(0.5, 0.5) node[above right] {$P_2$};
    \draw (R.south east) -- ++(0.5, -0.5) node[below right] {$P_1$};
\end{tikzpicture}
\hspace{1.2cm}
\begin{tikzpicture}
    \node[box] (L) at (0,-2.5) {};
    \node[box] (R) at (0.9,-3.4) {};

    \draw (L.south east) -- (R.north west) node[midway, above right] {$I$};
   
    \draw (L.north) -- ++(0, 0.5) node[above] {$A$};
    \draw (L.south west) -- ++(-0.5, -0.5) node[below left] {$X$};
   
    \draw (R.south west) -- ++(-0.5, -0.5) node[below left] {$P_1$};
    \draw (R.south east) -- ++(0.5, -0.5) node[below right] {$P_2$};
\end{tikzpicture}
\caption{Left: a two-node chain where $A$ is linked to $(X,I)$ and $I$ is linked to $(P_1,P_2)$. Right: the same network redrawn to emphasize a two-stage ``decomposition'' reading: $A$ relates to $(X,I)$, and $I$ relates to $(P_1,P_2)$. The two drawings are equivalent.}
\label{fig:expansion-nested}
\end{figure}
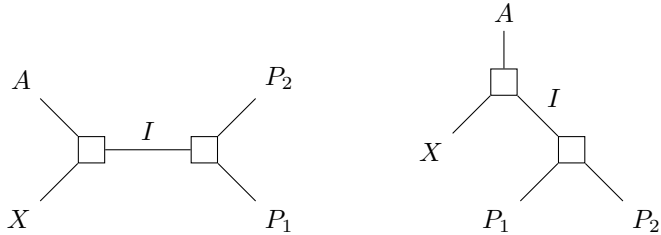
\begin{figure}[h]
\centering
\begin{tikzpicture}
    \node[box] (TOP) at (0, 0) {};
    \node[box] (BOT) at (0, -1.5) {};

    \draw (TOP.south) -- (BOT.north) node[midway, right] {$P_1'$};

    \draw (TOP.north west) -- ++(-0.5, 0.5) node[above left] {$A$};
    \draw (TOP.north east) -- ++(0.5, 0.5) node[above right] {$P_2$};

    \draw (BOT.south west) -- ++(-0.5, -0.5) node[below left] {$X$};
    \draw (BOT.south east) -- ++(0.5, -0.5) node[below right] {$P_1$};
\end{tikzpicture}
\hspace{1.2cm}
\begin{tikzpicture}
    \node[box] (L) at (0,-2.5) {};
    \node[box] (R) at (-0.9,-3.4) {};
    
    \draw (L.south west) -- (R.north east) node[midway, above left] {$P_1'$};
    
    \draw (L.north) -- ++(0, 0.5) node[above] {$A$};
    \draw (L.south east) -- ++(0.5, -0.5) node[below right] {$P_2$};
    
    \draw (R.south west) -- ++(-0.5, -0.5) node[below left] {$X$};
    \draw (R.south east) -- ++(0.5, -0.5) node[below right] {$P_1$};
\end{tikzpicture}
\caption{A pivoted version of \Cref{fig:expansion-nested}. Left: $A$ is now linked to $(P_1',P_2)$, and $P_1'$ is linked to $(X,P_1)$. Right: the same network redrawn to highlight the ``absorption'' interpretation: $P_1$ grows into $P_1'$ by incorporating $X$. The two drawings are equivalent.}
\label{fig:expansion-absorbing}
\end{figure}
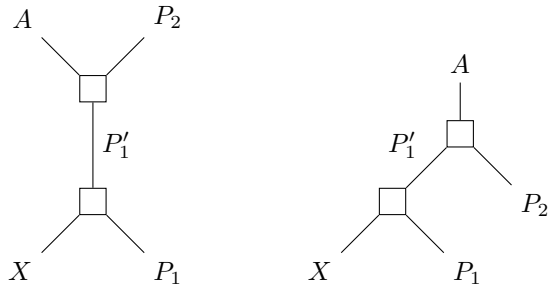

\paragraph{Pivoting is not guaranteed.}
Crucially, the right-hand configuration in \Cref{fig:pivot-move} is \emph{not} implied by the left-hand one in general.
One can build adjacent determinations that share a bridge $C$ but for which \emph{no} pivot concept $F$ exists.

\begin{theorem}[Non-universality of pivot moves]
\label{prop:pivot-not-always}
There exist adjacent $3$-way determinations sharing a bridge concept for which no pivot exists.

More concretely, fix $n$ sufficiently large and let $p=p_n$ be the \emph{first} prime in $[2^n,2^{n+1})$.\footnote{Any choice of a prime $p\in[2^n,2^{n+1})$ with $K(p)=O(\log n)$ would suffice. Choosing the first prime makes $p$ computable from $n$, hence $K(p)=O(\log n)$.}
Represent elements of $\mathbb{F}_p$ as binary strings of length $\lceil\log p\rceil+O(1)$, so that $n$ (hence $p_n$) is computable from the encoding length.

Then there exist
\[
A\in\mathbb{F}_p^\times,\qquad C\in\mathbb{F}_p,\qquad Z\in\mathbb{F}_p^\times
\]
such that
\[
K(A)\stackrel{+}{=}K(C)\stackrel{+}{=}K(Z)\stackrel{+}{=}n,
\qquad
K(A,C,Z)\;\ge\;3n-O(\log n),
\]
and if we define
\[
B:=ZA\in\mathbb{F}_p^\times,
\qquad
D:=Z+C\in\mathbb{F}_p,
\]
then:

\smallskip
\noindent\textbf{(i) (Two-node chain holds.)}
Both triples $(A,B,Z)$ and $(Z,C,D)$ are determinations, i.e.
\[
\forall i\in\{A,B,Z\},\qquad K(i\mid\{A,B,Z\}\setminus\{i\})\stackrel{+}{=}0,
\]
\[
\forall i\in\{Z,C,D\},\qquad K(i\mid\{Z,C,D\}\setminus\{i\})\stackrel{+}{=}0.
\]

\smallskip
\noindent\textbf{(ii) (No pivot.)}
There is \emph{no} binary string $X$ such that \emph{both} $(A,C,X)$ and $(B,D,X)$ are determinations.
Equivalently, there is no $X$ simultaneously satisfying
\[
K(X\mid A,C)\stackrel{+}{=}0,\qquad K(X\mid B,D)\stackrel{+}{=}0,\qquad K(B\mid D,X)\stackrel{+}{=}0,
\]
(and in particular, no $X$ satisfying the full six cross-recovery equalities).
\end{theorem}

\paragraph{Connection to common information.}
In \Cref{fig:pivot-move}, whenever the two adjacent determination nodes hold, the bridge $C$ is
computable from either side: $K(C\mid A,B)\stackrel{+}{=}0$ and $K(C\mid D,E)\stackrel{+}{=}0$.
By data processing, this implies $K(C)\stackrel{+}{\le} I((A,B);(D,E))$, so $C$ can be read as an
\emph{extractable} (materialized) piece of the mutual dependence between $(A,B)$ and $(D,E)$.
This matches the classical separation between common information and mutual information
\citep{gacs-korner-1973-commoninfo} and its algorithmic analogue of \emph{nonmaterializable}
mutual information \citep{romashchenko-2000-nonmaterializable}.
\Cref{prop:pivot-not-always} can be viewed as the same obstruction expressed at the level of pivot moves: a two-node chain may encode substantial dependence,
yet admit no pivot that would re-express it via a new bridge serving as an extractable common
information for the crossed pairs. Our proof follows the well-mixing-graph paradigm, using
spectral mixing to rule out all low-complexity candidate bridges, in the same spirit
as \citet{caillat-grenier-2024-well-mixing}.

Pivoting is best viewed as an \emph{opportunity} rather than a universal identity.
When a pivot exists, it enables a meaningful re-allocation of information between concepts (e.g.\ the ``$P_1$ absorbs $X$'' reading in \Cref{fig:expansion-absorbing}).
When it does \emph{not} exist, the failure itself is informative: it indicates that the two-node chain encodes a kind of dependency that cannot be re-expressed as two determinations sharing a different bridge without injecting substantial new information.

\subsection{Low-Complexity Determinations as Objects of Study}
\label{sec:low-complexity-det}

Previously we treated determination constraints as a \emph{structural} notion (any $n\!-\!1$ determine the last),
and studied how determinations can be completed and sometimes re-wired (pivot moves).
We now specify the main \emph{research object} we care about: determinations whose induced redundancy is small. 

To this end, we consider the multi-variable completion problem introduced above. For concreteness we use the case $n=3$ as a running example: we fix one component $A$ and optimize over the remaining two $(B,C)$ subject to forming a determination, although the definitions and results extend verbatim to general $n$.

\paragraph{A 3-way determination as a ``decomposition'' of information.}
Consider a 3-way determination $\pi=(A,B,C)$,
meaning
\[
K(A\mid B,C)\stackrel{+}{=}0,\qquad
K(B\mid A,C)\stackrel{+}{=}0,\qquad
K(C\mid A,B)\stackrel{+}{=}0.
\]
We read this as: the pair $(B,C)$ is a \emph{lossless representation} of $A$, while $B$ and $C$ are two ``parts'' that are consistent with $A$.

At the level of description lengths, any such representation must pay at least the intrinsic cost of $A$.
Indeed, since $A$ is computable from $(B,C)$,
\begin{equation}
K(A)\ \stackrel{+}{\le}\ K(B,C) \stackrel{+}{\le}\ K(B)+K(C).
\label{eq:BC-lb-A}
\end{equation}

\paragraph{Total complexity and excess information.}
We measure the ``size'' of a triple by the sum of its marginal complexities
\[
\Sigma(\pi)\ :=\ K(A)+K(B)+K(C).
\]
When $A$ is fixed and $\pi$ varies over determinations involving $A$, the term $K(A)$ is a constant,
so the meaningful quantity is the overhead incurred by storing $B$ and $C$ as \emph{separate} concepts.
We define the \emph{excess information} of $\pi$ as
\begin{equation}
\Delta(\pi)\ :=\ \Sigma(\pi)-2K(A)\ =\ K(B)+K(C)-K(A).
\label{eq:excess-information-def}
\end{equation}
By \eqref{eq:BC-lb-A}, $\Delta(\pi)\stackrel{+}{\ge}0$ for every determination $\pi=(A,B,C)$.

\paragraph{Excess information $\Delta$ captures the cost of an ``artificial split representation''.}
If one encodes the pair $(B,C)$ jointly, the right cost is $K(B,C)$.
Encoding them separately costs $K(B)+K(C)$, which wastes (algorithmic) mutual information:
\[
K(B)+K(C)-K(B,C)\ \stackrel{+}{=}\ I(B;C).
\]
Our excess \eqref{eq:excess-information-def} can be decomposed as
\[
\Delta(\pi)
\ =\ \underbrace{\bigl[K(B)+K(C)-K(B,C)\bigr]}_{\stackrel+=I(B;C)}
\ +\ \underbrace{\bigl[K(B,C)-K(A)\bigr]}_{\stackrel+= K(B,C|A)\ \text{since }K(A\mid B,C)\stackrel{+}{=}0}.
\]
Thus, $\Delta(\pi)$ upper-bounds the mutual information wasted by splitting, and additional information in $(B,C)$ beyond $A$, making precise the intuition: a ``good'' decomposition should \emph{avoid separating} two parts that share
substantial mutual information and avoid introducing new artifacts that were not in the original information.

\paragraph{There are only finitely many low-excess decompositions.}
Fix a string $A$ and an excess budget $L\ge 0$, and consider the family
\[
\mathcal{D}_A(L)\;:=\;\{(B,C):\ (A,B,C)\ \text{is a determination and}\ \Delta(A,B,C)\le L\}.
\]
Although $(B,C)$ range over infinitely many binary strings, the low-excess subset
$\mathcal{D}_A(L)$ is necessarily finite, and in fact its size grows at most
exponentially in $L$.

The key point is that low excess forces $(B,C)$ to be almost computable from $A$.
Indeed, every $\pi=(A,B,C)$ with $\Delta(\pi)\le L$ satisfies
\[
K(B,C\mid A)\ \stackrel{+}{\le}\ L.
\]

Now use the standard Kraft-counting fact for prefix complexity:
for any fixed condition $A$ and any integer $t$, the number of objects $u$ with
$K(u\mid A)\le t$ is at most $2^{t+1}-1$~\citep{LiVitanyi2019}.
Applying this to $u=\langle B,C\rangle$ and $t\approx L$ yields the crude but useful bound
\begin{equation}
|\mathcal{D}_A(L)|\ \le\ 2^{L+O(\log K(A))}.
\label{eq:finite-low-excess}
\end{equation}
In particular, for every fixed $A$ and $L$, there are only finitely many determinations
$\pi=(A,B,C)$ with excess information at most $L$, and each additional allowed bit of excess information can at most multiply the number of feasible decompositions by a constant factor.

This section specifies the \emph{object of study}: we care about determinations that realize a ``split representation'' of a given information object $A$ while incurring only small excess information $\Delta$ in~\eqref{eq:excess-information-def}. However, some \emph{methodological} questions remain open: how do we effectively \emph{find} such low-excess determinations, and how do we \emph{maintain or update} them when the underlying information $A$ changes? We address these questions in the next section.

 \section{Dialectics as an Optimization Problem}
\label{sec:dialectics-as-optimization}

Optimization is the common language of contemporary artificial intelligence and machine learning.
The typical methodological moves for training a classifier, fitting a generative model, or aligning a large language model are all describing the task as an optimization problem,
\[
\min_{\theta}\; L(\theta),
\]
and then searching effectively for low-loss solutions (often approximately) using available computation.

We will use the same framing for \emph{dialectics}, but with a different notion of ``loss'' and a different kind of decision variable.

\paragraph{Decision variables: \emph{free parts} of a determination diagram.}
Here the variables are not numeric parameters but \emph{concept strings} living on the edges of a determination network, and feasibility is enforced by determination constraints.

The simplest nontrivial dialectical situation is a single $3$-way determination $\pi=(I,P_1,P_2)$:
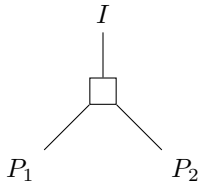
\begin{center}
\begin{tikzpicture}
    \node[box] (A) at (0,0) {};

    \draw (A.north) -- ++(0, 0.6)   node[above] {$I$};
    \draw (A.south west) -- ++(-0.6, -0.6) node[below left] {$P_1$};
    \draw (A.south east) -- ++(0.6, -0.6) node[below right] {$P_2$};
\end{tikzpicture}
\end{center}
Fixing $I$, any feasible solution $(P_1,P_2)$ is a \emph{lossless} split representation of $I$ in the sense that any missing component is recoverable from the other two.

\paragraph{Loss: excess information of a determination.}
Following \Cref{sec:low-complexity-det}, we quantify the redundancy induced by splitting via the excess information:
\begin{equation}
\Delta(\pi)\ :=\ K(P_1)+K(P_2)-K(I).
\label{eq:dialectic-excess}
\end{equation}

Therefore dialectics is just a particular optimization problem:
\begin{equation}
\label{eq:dialectic-core-opt}
\begin{aligned}
\min_{P_1,P_2\in\{0,1\}^\ast}\quad& K(P_1)+K(P_2)-K(I)
\\
\text{s.t.}\quad&
K(I\mid P_1,P_2)\stackrel{+}{=}0, K(P_1\mid I,P_2)\stackrel{+}{=}0, K(P_2\mid I,P_1)\stackrel{+}{=}0.
\end{aligned}
\end{equation}

Since each term $K(\cdot)$ in the loss is uncomputable but upper semicomputable, we will treat \eqref{eq:dialectic-core-opt} as an \emph{effective} optimization problem: for different feasible solutions, we compare \emph{constructive upper bounds} on the relevant complexities produced by concrete description methods (compression schemes, entropy encoding, probabilistic models, MDL-style two-part codes, prequential codes, etc.), and accept a feasible solution that decreases the current loss upper bound~\citep{Rissanen1978,Grunwald2007,Dawid1984}.

More generally, for a network of many determination nodes, or a determination node with many components, the decision variables are the strings on all the free edges, and global loss can be taken as the sum of complexities for each free concepts.

During optimization, multiple concepts coexist inside one network, and \emph{compete} for the right to explain different pieces of information by driving down the \emph{global loss}. The rest of this section develops practical local moves and comparison criteria for effectively solving this optimization. 

\subsection{From Global Objectives to Local Moves}
\label{sec:zeroth-order-dialectics}

Even though our decision variables are discrete strings and the objective involves the uncomputable quantity $K(\cdot)$, we can still optimize in a purely \emph{zeroth-order} fashion: we compare candidate \emph{feasible} solutions by their (constructive) description-length upper bounds and keep whichever is better. Concretely, a basic loop is:
\begin{enumerate}
\item generate a small set of candidate feasible solutions (via explicit constructions, heuristics, or exploration),
\item evaluate their losses (or the corresponding computable upper bounds),
\item retain the candidate with smaller loss.
\end{enumerate}
This style of search only needs \emph{relative improvement signals}: given two feasible candidates, decide which one decreases the objective.

In practice, feasible candidates produced by realistic construction methods are rarely ``independent samples'' from a vast space. They typically share most of their structure and differ by an interpretable modification, for instance, shifting one piece of information from one concept to another, or letting different concepts to absorb the same new information. We refer to such structured modifications as \emph{local moves}.

Local moves matter because they may admit \emph{local comparison rules}: instead of recomputing a global loss (or its upper bound) for all concepts, we can sometimes decide between two feasible solutions by only using Kolmogorov-complexity upper bounds for a much smaller set of strings.

\paragraph{The local moves.}
A very common local move is to \emph{relocate} a small information \emph{patch} $X$ from one concept to another.
Concretely, suppose the following two diagrams hold:
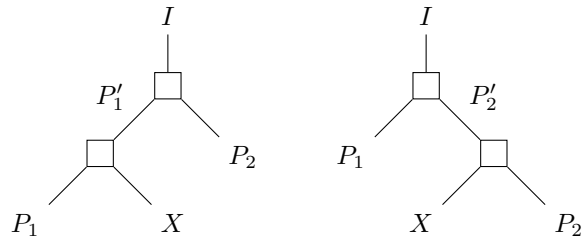
\begin{center}
\begin{tikzpicture}
    \node[box] (L) at (0,-2.5) {};
    \node[box] (R) at (-0.9,-3.4) {};
   
    \draw (L.south west) -- (R.north east) node[midway, above left] {$P_1'$};
   
    \draw (L.north) -- ++(0, 0.5) node[above] {$I$};
    \draw (L.south east) -- ++(0.5, -0.5) node[below right] {$P_2$};
   
    \draw (R.south west) -- ++(-0.5, -0.5) node[below left] {$P_1$};
    \draw (R.south east) -- ++(0.5, -0.5) node[below right] {$X$};
\end{tikzpicture}
\qquad
\begin{tikzpicture}
    \node[box] (L) at (0,-2.5) {};
    \node[box] (R) at (0.9,-3.4) {};

    \draw (L.south east) -- (R.north west) node[midway, above right] {$P_2'$};
  
    \draw (L.north) -- ++(0, 0.5) node[above] {$I$};
    \draw (L.south west) -- ++(-0.5, -0.5) node[below left] {$P_1$};
  
    \draw (R.south west) -- ++(-0.5, -0.5) node[below left] {$X$};
    \draw (R.south east) -- ++(0.5, -0.5) node[below right] {$P_2$};
\end{tikzpicture}
\end{center}
Intuitively, the left candidate asserts that $P_1'$ is a
larger version of $P_1$ that also includes $X$.
The right candidate is the symmetric alternative, where $P_2'$ includes $P_2$ and $X$. Informally, the pivot lets us \emph{slide} the
information in $X$ across the boundary between $P_1$ and $P_2$ without breaking feasibility, turning
$(P_1',P_2)$ and $(P_1,P_2')$ into two competing feasible solutions.

A natural situation where this pattern arises is when $I$ itself evolves. Suppose new
information arrives as a patch $X$, and the updated $I$ admits a lossless refinement into an
intermediate representation $I'$ which is then split into $(P_1,P_2)$:
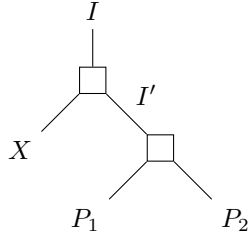
\begin{center}
\begin{tikzpicture}
    \node[box] (L) at (0,-2.5) {};
    \node[box] (R) at (0.9,-3.4) {};

    \draw (L.south east) -- (R.north west) node[midway, above right] {$I'$};
  
    \draw (L.north) -- ++(0, 0.5) node[above] {$I$};
    \draw (L.south west) -- ++(-0.5, -0.5) node[below left] {$X$};
  
    \draw (R.south west) -- ++(-0.5, -0.5) node[below left] {$P_1$};
    \draw (R.south east) -- ++(0.5, -0.5) node[below right] {$P_2$};
\end{tikzpicture}
\end{center}
Here $X$ is initially \emph{outside} the current split of $I'$ into $(P_1,P_2)$. If both sides can
expand to incorporate $X$ via the pivots above, then we recover exactly the local-move shown earlier: absorb $X$ into $P_1$ (producing $P_1'$) or absorb $X$ into $P_2$ (producing $P_2'$).

\paragraph{Deriving local relative comparison criterion via excess information.}
If we hold $I$ fixed and measure redundancy by the excess information objective
$\Delta = K(P_1)+K(P_2)-K(I)$ as in \eqref{eq:dialectic-excess}, then comparing the two candidates reduces to
comparing
\begin{equation}
K(P_1')+K(P_2)
\qquad\text{vs.}\qquad
K(P_1)+K(P_2'),
\label{eq:patch-direct-compare}
\end{equation}
since the $-K(I)$ term cancels.

The comparison \eqref{eq:patch-direct-compare} is conceptually clean, but it involves the marginal complexities
of the (potentially large) concepts $P_1,P_2,P_1',P_2'$.
A useful reduction is that the \emph{difference} between these two objective values can be rewritten using only
conditional complexities of the patch $X$.

\begin{lemma}[Complexity-difference identity in a determination]
\label{lem:det-difference-identity}
If $(A,B,C)$ is a $3$-way determination, then
\[
K(A)-K(B)\ \stackrel{+}{=}\ K(C\mid B)-K(C\mid A).
\]
\end{lemma}

Apply \Cref{lem:det-difference-identity} to the determination $(P_1',P_1,X)$:
\begin{equation}
K(P_1')-K(P_1)\ \stackrel{+}{=}\ K(X\mid P_1)-K(X\mid P_1').
\label{eq:patch-identity-1}
\end{equation}
Similarly, apply it to $(P_2',P_2,X)$:
\begin{equation}
K(P_2')-K(P_2)\ \stackrel{+}{=}\ K(X\mid P_2)-K(X\mid P_2').
\label{eq:patch-identity-2}
\end{equation}
Subtracting \eqref{eq:patch-identity-2} from \eqref{eq:patch-identity-1} yields
\begin{align}
&\bigl[K(P_1')+K(P_2)\bigr]-\bigl[K(P_1)+K(P_2')\bigr]\nonumber\\
=&\bigl[K(P_1')-K(P_1)\bigr]-\bigl[K(P_2')-K(P_2)\bigr]\nonumber\\
\stackrel{+}{=}&\bigl[K(X\mid P_1)+K(X\mid P_2')\bigr]-\bigl[K(X\mid P_1')+K(X\mid P_2)\bigr].
\label{eq:patch-local-criterion}
\end{align}
Consequently, deciding whether the patch should be absorbed on the $P_1$-side (producing $P_1'$) or on the
$P_2$-side (producing $P_2'$) is equivalent (up to logarithmic slack) to comparing
\begin{equation}
K(X\mid P_1)+K(X\mid P_2')
\qquad\text{vs.}\qquad
K(X\mid P_1')+K(X\mid P_2).
\label{eq:criterion-only-X}
\end{equation}

The point of \eqref{eq:criterion-only-X} is structural: although the global objective involves the marginal
complexities of the full concepts, the \emph{relative} preference between these two local rewires can be
expressed purely through how well the patch $X$ can be described under different side information.

\paragraph{Intuition: every concept ``yearns'' to expand by explaining new experience.}
\label{sec:dialectics-intuition}

Equation \eqref{eq:criterion-only-X} has a simple interpretation in the typical regime where the pivoted concepts $P_1',P_2'$
already ``contain'' $X$.

Indeed, in the absorption view, $P_1'$ is explicitly constructed to encode $X$ together with $P_1$.
Therefore one expects $K(X\mid P_1')$ to be small (often $\stackrel{+}{=}0$ in the idealized setting),
and likewise $K(X\mid P_2')$ small.
In contrast, $K(X\mid P_1)$ and $K(X\mid P_2)$ reflect how well the \emph{current} concepts can explain/compress the new patch.

Thus the decision is dominated by the ``explanatory power'' terms $K(X\mid P_1)$ versus $K(X\mid P_2)$
(up to the bookkeeping terms involving $P_1',P_2'$).
The concept that offers the shorter conditional description of $X$ is the concept that \emph{best assimilates} $X$.

This motivates the dialectical metaphor:
treat $P_1$ and $P_2$ as two ``living'' concepts that compete to incorporate (interpret, explain, compress) new experience $X$.
The winner expands, the loser retracts (or, more precisely, it does not get to claim $X$ as part of its internal structure).

\paragraph{An effective dialectical step: isolate a patch and let concepts compete.}
The discussion above suggests a generic computational template:
\begin{enumerate}
\item \emph{Isolate a patch.} Isolate a \emph{patch} $X$ that is up for reinterpretation: it may be genuinely new information (a new image region, a newly read sentence), or it may be carved out from an existing concept representation (a span in a document, a subset of pixels, a residual in a model).
\item \emph{Generate candidate feasible solutions.} Starting from a feasible split $(P_1,P_2)$ for a fixed $I$, generate two feasible alternatives: absorb $X$ into $P_1$ (yielding $(P_1',P_2)$) or absorb $X$ into $P_2$ (yielding $(P_1,P_2')$). When the corresponding pivot constructions exist, these are two competing feasible solutions.
\item \emph{Evaluate and retain the better solution.} Decide ``who gets $X$'' by their (constructive) conditional description lengths for $X$ (cf. \eqref{eq:criterion-only-X}), rather than recomputing global complexities of the full concepts, and retain the feasible solution that yields the smaller bound.
\end{enumerate}
Operationally, this is a ``dialectical'' competition: two concepts are allowed to spend computation to produce a shorter description of the same $X$ given their respective side information, and whichever side achieves the smaller bound wins the right to incorporate $X$.

\paragraph{Approximate monotonicity: a Lyapunov view for the \emph{known} objective.}
If we always accept only moves that strictly decrease our \emph{current} computable upper bound on the excess information objective $\Delta=K(P_1)+K(P_2)-K(I)$ (cf.\ \eqref{eq:dialectic-excess}), then that tracked upper bound is monotone decreasing and plays the role of a Lyapunov function for the optimization dynamics. However, this does \emph{not} imply that the unknown true $\Delta$ is monotone, nor that assignments of $X$ are irreversible: because our bounds can improve later possibly due to more computation, a patch that ``belongs to'' $P_1$ under the current coding method may later be re-assigned when a better explanation is discovered. Also, considering that all equalities above are only up to logarithmic slack, a simple heuristic could improve robustness with a margin: only commit the move when the code-length advantage is larger than the expected slack/estimation noise, so that dialectical steps become stable when the evidence is strong and remain revisable only in ambiguous/near-ties regimes.

\paragraph{Why realizability/effectivity matters.}  Pure existence results are not enough for epistemic or algorithmic purposes. For instance, the constraint system always admits the trivial feasible split $P_1=I,\,P_2=\epsilon$ (``the empty string''), which drives $\Delta$ near zero \emph{in principle} but does not yield a useful articulation of $I$ into independently meaningful parts. In our \emph{effective} (constructive) setup: the ``concepts'' must come with concrete description methods that actually achieve the claimed upper bounds. Otherwise, a nominal assignment such as ``$P_1$ is the whole world'' is not stable under competition: another concept $P_2$ can partially compress the data and thereby win an explanatory territory from $P_1$.
Therefore, dialectics is resource-bounded: concepts remain distinct only if, under available computation, neither side can produce a decisively shorter code for the other side, and the boundaries of one concept become stable precisely when they are backed by real compression advantages that survive adversarial re-coding attempts by competing concepts.

\paragraph{Limited compressors already yield rich structure.}
Finally, the framework does not require near-optimal Kolmogorov codes to be useful.

Even very limited code families already induce quite useful and human-recognizable structure. For instance, coding with a simple Gaussian model on vectors already leads to local relative comparison criterion corresponding to familiar mechanisms such as $k$-means/EM-style assignments: a cluster ``wins'' a vector precisely when it can describe it with fewer bits under its current statistics~\citep{MacQueen1967,Lloyd1982,Dempster1977}. This is the general phenomenon: crude models are still strong enough to expose large, stable determinations (textures vs.\ boundaries, topics vs.\ background, common patterns vs.\ outliers), because those determinations create large code-length gaps that dominate the $O(\log)$ slack and estimation noise. As we tighten the available upper bounds (by enlarging the model class, improving predictors, or spending more compute) the same competition gives concepts more expressive ``arguments'' to refine the boundary or reveal finer determinations.

Seen this way, the dialectical framing does not replace existing practice; it reinterprets much of it as a family of practical approximations to the same underlying objective. Different communities have implicitly instantiated different ``glue'' constraints and different code families, but many successful algorithms can be read as implementing local competitive moves of the form ``who gets to explain this piece?'' under a tractable description-length surrogate. More examples are discussed in \Cref{sec:related_work}. This perspective suggests a systematic program: (i) choose a realizable code family, (ii) define feasible pivot that preserve determination constraints to construct different feasible solutions, and (iii) allocate computation to let competing concepts compress certain information, keeping only moves that decrease the tracked code length.

\subsection{Dialectics with Grounds}
\label{sec:grounded-dialectics}
So far the two sides of a split were perfectly symmetric: swapping the labels of $P_1$ and $P_2$
does not change either feasibility (determination constraints) or the excess-information objective.
In many settings, however, we need to \emph{refer to} (and later \emph{re-identify}) a particular side of the split (e.g.\ a system that must keep concept identities consistent across time).
To break this exchangeability we introduce \emph{grounds}, which are asymmetric side information that anchors each side.

\paragraph{A grounded split diagram.}
We attach two auxiliary strings $\hat A,\hat B$ (the grounds) to the two sides:
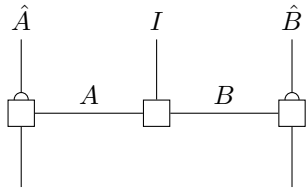
\begin{center}
\begin{tikzpicture}
    \node[box] (L) at (0,0) {};
    \node[box] (M) at (1.8,0) {};
    \node[box] (R) at (3.6,0) {};

    \draw[arc] (L.north) -- ++(0, 0.8) node[above] {$\hat{A}$};
    \draw[arc] (R.north) -- ++(0, 0.8) node[above] {$\hat{B}$};
    \draw (M.north) -- ++(0, 0.8) node[above] {$I$};

    \draw (L.east) -- (M.west) node[midway, above] {$A$};
    \draw (M.east) -- (R.west) node[midway, above] {$B$};

    \draw (L.south) -- ++(0, -0.8);
    \draw (R.south) -- ++(0, -0.8);
\end{tikzpicture}
\end{center}
The intended reading is: $(A,B)$ is a lossless split representation of $I$ (a 3-way determination),
but each side is additionally encouraged to stay close to its ground via a conditional description penalty.

\paragraph{Grounded dialectics as an optimization problem.}
Fix the underlying information object $I$ and grounds $(\hat A,\hat B)$.
We define the grounded split by the constrained optimization
\begin{equation}
\begin{aligned}
\min_{A,B\in\{0,1\}^\ast}\quad& K(A \mid \hat A) + K(B \mid \hat B) \\
\text{s.t.}\quad&
K(I \mid A,B) \stackrel{+}{=}0,\quad K(B \mid A,I) \stackrel{+}{=}0,\quad K(A \mid B,I) \stackrel{+}{=}0.
\end{aligned}
\label{eq:grounded-split-opt}
\end{equation}
Write $(A^{\star},B^{\star})$ for an optimal solution.

\paragraph{What ``pin down'' means.}
Given $(I,\hat A,\hat B)$, the grounds \emph{pin down} a split in the sense that they specify the designated output of
\eqref{eq:grounded-split-opt} (the Kolmogorov-optimal split, up to $\stackrel{+}{=}$); in practice, one can also
take ``pin down'' to mean the output of a \emph{publicly agreed} (deterministic) optimization protocol and compute budget,
so that the same inputs yield the same (approximate) split with no ambiguity.

\paragraph{Function of the grounds.}
The grounds $\hat A,\hat B$ serve two roles:
\begin{enumerate}
\item \textbf{Disambiguation:} they make the two concepts non-exchangeable, so ``which side is which'' becomes meaningful
      (i.e. the split is pinned down under the chosen optimality notion / protocol).
\item \textbf{Identification:} they provide a compact handle (e.g. a small seed set, a few prototypical examples, a short specification) by which one agent can recover one concept without storing whole concept.
\end{enumerate}

\paragraph{There always exists (at least trivial) anchoring that pins down any target split.}
For any target feasible pair $(A',B')$, one may take $\hat A=A'$, $\hat B=B'$.
Then $(A',B')$ achieves objective $K(A'\mid \hat A)+K(B'\mid \hat B)\stackrel{+}{=}0$ and hence is optimal. Under a fixed tie-breaking rule / common protocol, this pins down a unique reproducible output.

\paragraph{Ground complexity upper-bounds attainable excess information.}
A useful qualitative principle is that small grounds can only identify low-excess determinations.
Formally:

\begin{theorem}[Ground complexity controls excess]
\label{thm:ground-bounds-excess}
Let $(A^{\star},B^{\star})$ be optimal for \eqref{eq:grounded-split-opt}. Then
\[
K(\hat{A})+K(\hat{B})\ \stackrel{+}{\ge}\ K(A^{\star})+K(B^{\star})-K(I).
\]
\end{theorem}

Intuitively, if a concept can be specified by a very small seed (e.g. one human can quickly learn concepts like ``guacamole''
vs.\ ``non-guacamole'' via a handful of examples), then the corresponding grounded split is forced to have low excess
information. We will exploit this idea for communication and alignment of concepts between agents.

\subsection{Communication: Delivering Concepts via Small Seed Sets}
\label{sec:communication-seeds}
A final motivation for grounded dialectics is communication.
If two agents share the same dialectical procedure (the same ``rules of competition''),
then to communicate a concept it can be enough to transmit a \emph{small seed set} rather than the full extension.

\paragraph{Example (cat video).}
Suppose agents $A$ and $B$ watch the same cat video.
Agent $A$ can indicate the cat in the first frame (a small seed: a boundary, a pointer or a short description).
Agent $B$, using the shared dialectical mechanism, can then ``grow'' the cat concept across the entire video, tracking it through pose changes and occlusions in other frames.
The seed is vastly smaller than the total set of cat pixels across all frames.

\paragraph{Sender and receiver roles.}
\begin{itemize}
\item The \textbf{sender} selects representative seeds (positive/negative examples, or anchors $\hat A,\hat B$)
so that the intended concept becomes the stable outcome of dialectical optimization under the shared protocol.
\item The \textbf{receiver} runs dialectical growth starting from the seed constraints and reconstructs the concept's extension.
\end{itemize}

\subsection{The Explicit-Boundary Regime}
\label{sec:boundary-regime}
A large class of ``segmentation-like'' problems comes with an \emph{explicit boundary object}:
a curve separating foreground/background, a binary mask, a set of polygon vertices, a span $(\texttt{start},\texttt{end})$ inside a document, a time-frequency mask in audio, etc.
In these cases the split is not only an abstract factorization of information but also a \emph{boundary-mediated decomposition}.

We model this regime by four strings:
\[
I \quad\text{(the whole information object)},\qquad
P_1,\ P_2 \quad\text{(two parts)},\qquad
B \quad\text{(a boundary)}.
\]
The intended semantics is that $B$ specifies \emph{where} the split happens, while $P_1,P_2$ carry the corresponding content.\footnote{In applications, $P_i$ often includes both \emph{content} and \emph{addressing} (coordinates / indices). This is exactly why $B$ may be recoverable from a single part.}

\paragraph{Boundary information assumptions.}
We isolate three boundary information assumptions:
\begin{enumerate}
\item[\textbf{A1.}] (Boundary visible from either side.)
\begin{equation}
K(B\mid P_i)\stackrel{+}{=}0,\qquad i\in\{1,2\}.
\label{eq:A1}
\end{equation}
\item[\textbf{A2.}] (Parts are pure ``restrictions'' of the whole given the boundary.)
\begin{equation}
K(P_i\mid I,B)\stackrel{+}{=}0,\qquad i\in\{1,2\}.
\label{eq:A2}
\end{equation}
\item[\textbf{A3.}] (The two parts cover the whole.)
\begin{equation}
K(I\mid P_1,P_2)\stackrel{+}{=}0.
\label{eq:A3}
\end{equation}
\end{enumerate}
When \eqref{eq:A1}--\eqref{eq:A3} hold, the triple $(I,P_1,P_2)$ is a determination (cf.\ \Cref{def:determination}):
indeed, from $I$ and (say) $P_2$ we recover $B$ by \eqref{eq:A1} and then recover $P_1$ by \eqref{eq:A2},
so $K(P_1\mid I,P_2)\stackrel{+}{=}0$, and symmetrically for $P_2$.

\paragraph{Why this regime is useful.}
Dialectics, as formulated in \eqref{eq:dialectic-core-opt}, cares about the redundancy cost
\[
\Delta \;=\; K(P_1)+K(P_2)-K(I),
\]
but in the explicit-boundary regime this quantity admits multiple equivalent decompositions
that separate ``\emph{boundary ambiguity}'' from ``\emph{content entanglement}''.
This is exactly the situation where local ``boundary competition'' rules become natural and interpretable.

\paragraph{Seven equivalent decompositions of the excess information.}
\label{sec:delta-seven}

We collect a family of ``excess information'' that all reduce (under boundary information assumptions)
to the same objective, but highlight different mechanisms by which a split can be wasteful.

Define the following seven expressions:
\begin{align}
\Delta^{(1)} &:= I(P_1;P_2)+K(B\mid I),
\label{eq:Delta1}\\
\Delta^{(2)} &:= K(B)+I(P_1;P_2\mid B)+K(B\mid I),
\label{eq:Delta2}\\
\Delta^{(3)} &:= K(B)+K(P_1\mid B)+K(P_2\mid B)+K(B\mid I)-I\big((P_1,P_2);I\mid B\big),
\label{eq:Delta3}\\
\Delta^{(4)} &:= K(P_1, B)+K(P_2, B)-K(I),
\label{eq:Delta4}\\
\Delta^{(5)} &:= K(P_1)+K(P_2)-K(I),
\label{eq:Delta5}\\
\Delta^{(6)} &:= I(P_1;P_2\mid B)+I(P_1;P_2\mid I)+K(B),
\label{eq:Delta6}\\
\Delta^{(7)} &:= K(I\mid B)+K(B\mid I)-K(P_1\mid P_2)-K(P_2\mid P_1)+K(B).
\label{eq:Delta7}
\end{align}

Some equivalent forms of excess information highlight different ways a split can go wrong.
\begin{itemize}
\item $\Delta^{(1)}$ grows when $P_1$ and $P_2$ share a lot of algorithmic content and/or when $B$ is not (nearly) determined by $I$.

\item $\Delta^{(2)}$ grows when the boundary is complex, fails to decouple the parts, or is underdetermined by the whole.

\item $\Delta^{(3)}$ grows when the sides remain complex even after fixing $B$, and/or when $(P_1,P_2)$ carries little information about $I$ once $B$ is known.

\item $\Delta^{(4)}$ grows when boundary/interface structure is effectively present on both sides.

\item $\Delta^{(6)}$ grows when $B$ does not disentangle the parts (large $I(P_1;P_2\mid B)$) and/or when the split introduces shared degrees of freedom not fixed by $I$ (large $I(P_1;P_2\mid I)$).

\item $\Delta^{(7)}$ becomes large when each part is easy to reconstruct from the other (small $K(P_1\mid P_2)$ and $K(P_2\mid P_1)$), especially if the boundary is costly.
\end{itemize}

Under the explicit-boundary assumptions \eqref{eq:A1}-\eqref{eq:A3}, these objectives coincide.
More precisely, we show the following implication graph in \Cref{fig:delta-relations} (labels indicate which assumption is used).

\begin{figure}[h]
\centering
\begin{tikzpicture}[
    x=2.5cm, y=2.2cm, 
    box/.style={draw, thin, rounded corners, minimum height=1.2cm, minimum width=1.6cm, font=\large, fill=white, align=center},
    sym/.style={fill=white, inner sep=1pt, font=\normalsize},
    lbl/.style={font=\scriptsize\sffamily, text=gray} 
]

\def\pe{\stackrel{\scriptscriptstyle+}{=}}
    \def\pg{\stackrel{\scriptscriptstyle+}{\ge}}
    \def\pl{\stackrel{\scriptscriptstyle+}{\le}}

\node[box] (d1) at (0,0) {$\Delta^{(1)}$};
    \node[box] (d2) at (0,1) {$\Delta^{(2)}$};
    \node[box] (d3) at (1,1) {$\Delta^{(3)}$};
    \node[box] (d4) at (2,1) {$\Delta^{(4)}$};
    \node[box] (d5) at (2,0) {$\Delta^{(5)}$};
    \node[box] (d6) at (-1.5,1) {$\Delta^{(6)}$};
    \node[box, minimum width=3.2cm] (d7) at (-1.5,0) {$\Delta^{(7)} \pe 2\Delta^{(1)} - \Delta^{(5)}$};

\draw (d1) -- node[sym]{$\pe$} node[lbl, left=4pt] {A1} (d2);
    \draw (d2) -- node[sym]{$\pe$} node[lbl, above=4pt]{A2} (d3);
    \draw (d3) -- node[sym]{$\pe$} node[lbl, above=4pt]{A3} (d4);
    \draw (d5) -- node[sym]{$\pe$} node[lbl, right=4pt]{A1} (d4);

\draw ([yshift=3mm]d6.east) -- node[sym]{$\pg$} node[lbl, above=4pt]{A1} ([yshift=3mm]d2.west);
    \draw ([yshift=-3mm]d6.east) -- node[sym]{$\pl$} node[lbl, below=6pt]{A2} ([yshift=-3mm]d2.west);

\draw[dashed, gray] (d5.south) -- ++(0, -0.5cm) -| (d7.south);
    \draw[dashed, gray] (d1) -- (d7);

\end{tikzpicture}
\caption{Relations among the seven split-overhead expressions. Under the explicit-boundary assumptions A1--A3, all $\Delta^{(k)}$ coincide, under partial assumptions we might obtain one-sided comparisons. Proofs are in \Cref{sec:proofs}.}
\label{fig:delta-relations}
\end{figure}
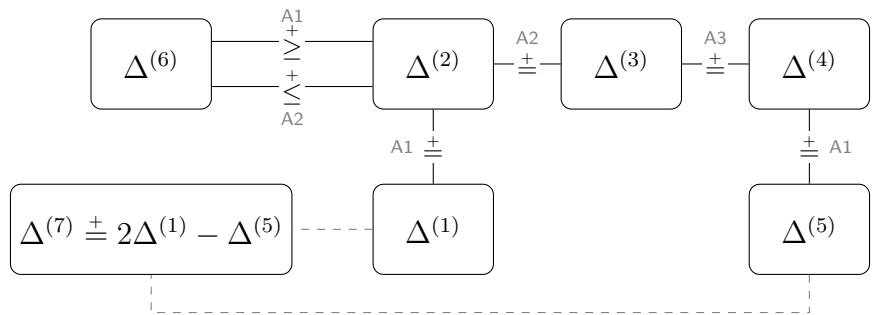

\paragraph{Information accounting diagrams.}
\label{sec:info-accounting-diagrams}

The algebraic relations behind \Cref{fig:delta-relations} can be summarized by a compact ``information accounting'' picture.
\Cref{fig:info-decomposition} is an illustration of how boundary-related and content-related information distribute across a two-part split.
\Cref{fig:info-components} provides a complementary view to highlight which regions of the accounting diagram correspond to which \emph{components} ($P_1,P_2,B,I$).
These diagrams provide a compact way to reason about which terms would be expected to move when we change (i) the boundary $B$ or (ii) a small patch of information is assigned to $P_1$ versus $P_2$.

\begin{figure}[h]
\centering
\begin{tikzpicture}

\coordinate (Top)   at (0, 4.5);
\coordinate (Left)  at (-5, 0);
\coordinate (Right) at (5, 0);
\coordinate (BottomCenter) at (0, 0);

\coordinate (MidLeft)  at ($(Left)!.5!(Top)$);
\coordinate (MidRight) at ($(Right)!.5!(Top)$);

\coordinate (SplitL) at ($(BottomCenter)!.6!(MidLeft)$);
\coordinate (SplitR) at ($(BottomCenter)!.6!(MidRight)$);

\draw (Left) -- (Right) -- (Top) -- cycle;
\draw (MidLeft) -- (MidRight) -- (BottomCenter) -- cycle;
\draw (SplitL) -- (SplitR);

\node at (barycentric cs:Top=1,MidLeft=1,MidRight=1) {$K(B\mid I)$};
\node at (barycentric cs:Left=1,BottomCenter=1,MidLeft=1) {$K(P_1\mid P_2)$};
\node at (barycentric cs:Right=1,BottomCenter=1,MidRight=1) {$K(P_2\mid P_1)$};
\node at (barycentric cs:SplitL=1,SplitR=1,BottomCenter=1) {$I(P_1; P_2\mid B)$};

\coordinate (TrapTopMid) at ($(MidLeft)!.5!(MidRight)$);
\coordinate (TrapBotMid) at ($(SplitL)!.5!(SplitR)$);
\node at ($(TrapTopMid)!.5!(TrapBotMid)$) {$I(B; I)$};

\end{tikzpicture}
\caption{An information-accounting decomposition for the explicit-boundary regime. The labels indicate terms that repeatedly appear in the equivalent overhead objectives $\Delta^{(k)}$.}
\label{fig:info-decomposition}
\end{figure}
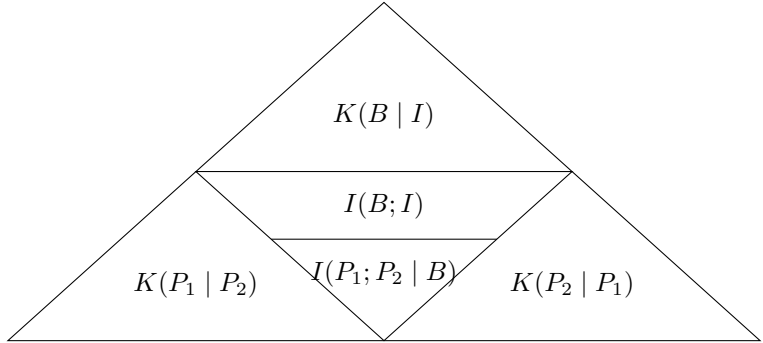

\begin{figure}[h]
\centering
\begin{tikzpicture}[
    >=latex,
    semithick,
    font=\scriptsize,
    shaded region/.style={fill=gray!30},
    subcaption/.style={text width=4cm, align=center, yshift=-0.5cm, font=\small}
]
    \newcommand{\DrawDiagram}[1]{
        \coordinate (Top) at (0, 2);
        \coordinate (Left) at (-2, 0);
        \coordinate (Right) at (2, 0);

        \coordinate (MidLeft) at ($(Top)!0.5!(Left)$);
        \coordinate (MidRight) at ($(Top)!0.5!(Right)$);
        \coordinate (MidBottom) at (0, 0);

        \coordinate (SplitL) at ($(MidBottom)!0.6!(MidLeft)$);
        \coordinate (SplitR) at ($(MidBottom)!0.6!(MidRight)$);

        #1

        \draw (Left) -- (Right) -- (Top) -- cycle;
        \draw (MidLeft) -- (MidRight) -- (MidBottom) -- cycle;
        \draw (SplitL) -- (SplitR);
    }

\begin{scope}[shift={(0, 0)}]
        \DrawDiagram{
            \fill[shaded region] (Left) -- (Top) -- (MidRight) -- (MidBottom) -- cycle;
        }
        \node[subcaption] at (0, -0.2) {(a) Component $P_1$};
    \end{scope}

\begin{scope}[shift={(5, 0)}]
        \DrawDiagram{
            \fill[shaded region] (Right) -- (Top) -- (MidLeft) -- (MidBottom) -- cycle;
        }
        \node[subcaption] at (0, -0.2) {(b) Component $P_2$};
    \end{scope}

\begin{scope}[shift={(0, -3.8)}]
        \DrawDiagram{
            \fill[shaded region] (Top) -- (MidLeft) -- (MidRight) -- cycle;
            \fill[shaded region] (MidLeft) -- (MidRight) -- (SplitR) -- (SplitL) -- cycle;
        }
        \node[subcaption] at (0, -0.2) {(c) Component $B$};
    \end{scope}

\begin{scope}[shift={(5, -3.8)}]
        \DrawDiagram{
            \fill[shaded region] (Left) -- (MidLeft) -- (MidRight) -- (Right) -- cycle;
        }
        \node[subcaption] at (0, -0.2) {(d) Component $I$};
    \end{scope}

\end{tikzpicture}
\caption{A shading convention for the accounting diagram: each component can be viewed as a particular ``region'' of the shared information geometry. }
\label{fig:info-components}
\end{figure}
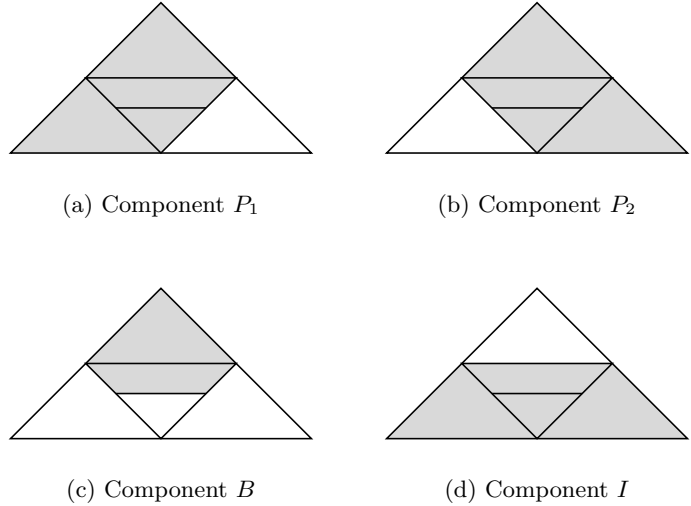
\subsubsection{Dialectical Criteria as Boundary Competition}
\label{sec:dialectics-criterion-boundary}

A key point in the explicit-boundary regime is that the \emph{decision variable is the boundary} $B$:
once $B$ is chosen, the parts $P_1,P_2$ are induced as (algorithmic) restrictions of $I$ to the two sides,
so $P_i$ should be viewed as functions of $B$ rather than independent variables.

\paragraph{A local ``contested block'' move.}
Consider a small information block $X$ near the current boundary.
Assume the total information is partitioned into three disjoint pieces $P_1,P_2,X$ (so $X$ is not already inside $P_1$ or $P_2$) and there are two competing ways to allocate $X$:
\begin{gather*}
\text{(1) assign $X$ to side 1:}\quad (P_1',P_2)\ \text{with boundary }B_1,
\\
\text{(2) assign $X$ to side 2:}\quad (P_1,P_2')\ \text{with boundary }B_2,
\end{gather*}
where $P_1'$ denotes $P_1$ expanded by including $X$, and $P_2'$ denotes $P_2$ expanded by including $X$.
We evaluate the split overhead using the base objective
\[
\Delta^{(5)}\ :=\ K(P_1)+K(P_2)-K(I),
\]
so the two candidate overheads are
\[
\Delta_1\stackrel{+}{=}K(P_1')+K(P_2)-K(I),
\qquad
\Delta_2\stackrel{+}{=}K(P_1)+K(P_2')-K(I).
\]
Therefore the \emph{dialectical local preference} is determined by
\begin{equation}
\Delta_1-\Delta_2
\ \stackrel{+}{=}\ 
\bigl[K(P_1')-K(P_1)\bigr]\;-\;\bigl[K(P_2')-K(P_2)\bigr].
\label{eq:delta-diff-basic}
\end{equation}
This already has a direct interpretation: each side ``bids'' the marginal description-length increase needed to absorb $X$.

\paragraph{From marginal costs to conditional predictability.}
Using standard symmetry-of-information manipulations (applied to $(P_1,X,P_1')$ and $(P_2,X,P_2')$),
\eqref{eq:delta-diff-basic} can be rewritten as
\begin{equation}
\Delta_1-\Delta_2
\ \stackrel{+}{=}\ 
\bigl[K(X\mid P_1)-K(X\mid P_2)\bigr]
\;+\;
\bigl[K(X\mid P_2')-K(X\mid P_1')\bigr].
\label{eq:delta-diff-X}
\end{equation}
The first bracket is the intuitive ``content term'': $X$ should go to the side that predicts/compresses it better.
The second bracket is a \emph{correction} about how hard it is to \emph{locate/extract} $X$ inside the enlarged parts, because $P_1',P_2'$ themselves already contain $X$.

\paragraph{Turning the correction into a boundary term.}
To connect this correction to the boundary, we use the following local version of boundary information assumptions:
both candidate decompositions $(P_1',P_2,B_1)$ and $(P_1,P_2',B_2)$ satisfy the boundary regime constraints
(A1-A3) with their respective boundaries, and moreover the boundary inside the enlarged patch is
recoverable once we know \emph{which sub-block} is $X$ (and conversely, $X$ is recoverable once we know the boundary cut).
Formally, it suffices to assume (and symmetrically for the other side) the two short recoverability statements
\[
K(P_2,X\mid P_2',B_1)\stackrel{+}{=}0,
\qquad
K(B_1\mid P_2',X)\stackrel{+}{=}0,
\]
which imply the equivalence
\[
K(X\mid P_2')\stackrel{+}{=}K(B_1\mid P_2'),
\qquad
K(X\mid P_1')\stackrel{+}{=}K(B_2\mid P_1').
\]
Substituting these into \eqref{eq:delta-diff-X} yields the advertised ``boundary competition'' form:
\begin{equation}
\Delta_1-\Delta_2
\ \stackrel{+}{=}\ 
\bigl[K(X\mid P_1)-K(X\mid P_2)\bigr]
\;+\;
\bigl[K(B_1\mid P_2')-K(B_2\mid P_1')\bigr].
\label{eq:delta-diff-final}
\end{equation}

\paragraph{Reading \eqref{eq:delta-diff-final}.}
The dialectical criterion is now explicit:
\begin{itemize}
\item the \emph{content bid} $\;K(X\mid P_1)-K(X\mid P_2)\;$ prefers assigning $X$ to the side that makes it more compressible;
\item the \emph{boundary bid} $\;K(B_1\mid P_2')-K(B_2\mid P_1')\;$ penalizes the assignment that forces a more complex (or less ``visible'') boundary within the enlarged parts.
\end{itemize}
In one simplification where boundary bookkeeping is negligible (or comparable across the two options),
\eqref{eq:delta-diff-final} reduces to the clean rule:
\[
\Delta_1\le \Delta_2\quad\Longleftrightarrow\quad K(X\mid P_1)\ \stackrel{+}{\le}\ K(X\mid P_2),
\]
i.e. $X$ is claimed by the side that best explains it. \section{Related Work}
\label{sec:related_work}

\paragraph{Looking back: compression as inference and as a learning objective.}
The idea that \emph{good explanations compress} has a long history across information theory, statistics, and learning theory.
Minimum Description Length (MDL) formalizes model selection as minimizing a two-part (model + data-given-model) code length
\citep{Rissanen1978,Grunwald2007}. Closely related, Minimum Message Length (MML) derives a similar objective from Bayesian
coding arguments \citep{WallaceBoulton1968}, and the \emph{prequential} view evaluates predictors by sequential predictive
log-loss (equivalently, a sequential code length) \citep{Dawid1984}.
These traditions already contain many proto-versions of what we call ``dialectics'':
a competition between alternative explanations, adjudicated by their achieved description lengths under a shared coding protocol.

\paragraph{Clustering by compression (distance-based) and how we differ.}
A prominent compression-based approach to unsupervised grouping is to define a \emph{pairwise} similarity via
information/compression distances. The pioneering line commonly known as \emph{Clustering by Compression}
introduces the (ideal) normalized information distance (NID) and its practical surrogate, the normalized compression distance
(NCD), and then performs clustering from the induced pairwise distance matrix \citep{Bennett1998,CilibrasiVitanyi2005}.
This is closely related in spirit to our goal, that is discovering ``natural'' groupings via compression, but it differs in two
structural ways.

First, NCD-style methods are fundamentally \emph{pairwise}: one computes a distance between every pair of objects (or uses a
subset thereof), and then runs a downstream clustering procedure. In contrast, dialectics is formulated as a \emph{global}
coding problem in which multiple concepts/segments compete to absorb contested information so as to reduce a shared
description-length objective. Second, NCD is \emph{normalized} (e.g., by $\max\{C(x),C(y)\}$), which makes it a robust
similarity measure across scales but also decouples it from the ``who gets which pieces'' allocation dynamics that our
framework emphasizes. In our setting, what gets compressed is not merely the concatenation of two objects, but the
\emph{contested region/patch} (or a newly arriving sample) under competing side information, the comparison signal is a
\emph{conditional} code-length (an operational surrogate for $K(X\mid P_i)$), not a symmetric normalized distance.

\paragraph{$k$-means as the crudest Euclidean dialectics.}
A useful sanity check is that, under a very restricted family of codes, dialectical competition reduces to classical
assignment rules.
If each cluster is represented by a simple spherical Gaussian code (or equivalently, a squared-error distortion model plus a
shared variance/constant), then deciding which cluster should absorb a point is equivalent to choosing the cluster that yields
the smallest incremental code length, which, after dropping constants, becomes the nearest-centroid rule of $k$-means
\citep{MacQueen1967,Lloyd1982}. In this view, MacQueen's \emph{online} $k$-means is literally ``$K$ segments competing for the
next sample'': each incoming point is claimed by the cluster that explains it best under the current sufficient statistics.
Our contribution is \emph{not} to re-derive $k$-means, but to allow using the same competitive-coding template with much richer code
families and (importantly) with more general \emph{reallocation moves} (patch/boundary competition) that go beyond pointwise
assignment in a fixed geometry.

\paragraph{More general clustering on collections.}
When the input is a data \emph{collection} $\{x_i\}_{i=1}^N$, our ``segmentation'' viewpoint reduces to a clustering problem:
we seek an assignment of items into groups (and, implicitly or explicitly, a within-group code/model) that minimizes a total
description-length objective. 
In this regime, a large family of existing methods can be read as \emph{special cases} of our
dialectical formulation once we (i) choose a concrete, computable code family to upper-bound Kolmogorov terms and (ii) choose
a restricted set of feasible \emph{local moves} for searching over partitions.

A representative early example is \citet{MaDerksenHongWright2007}, who study segmentation of multivariate mixed data via lossy
data coding. Their setting treats the data essentially as an \emph{unordered point set} (without additional geometry or
sequence structure) and optimizes a coding-length surrogate (e.g., Gaussian rate-distortion approximations) over
partitions. Algorithmically, the search is carried out by local merge/split or pairwise descent-style updates, which also change cluster number, rather than explicit ``two-sided debates''.

At the other end of the spectrum, recent deep generative clustering methods instantiate the idea that \emph{each cluster is a
model} at neural scale. For example, \citet{AdipoetraMartin2025} attach a generative model (e.g., a VAE) to each cluster and
alternate between updating assignments (by comparing per-cluster likelihood/ELBO scores for each sample) and updating model
parameters to maximize the corresponding ELBO given the current assignments. Although written in the standard
EM/coordinate-ascent template, this is naturally interpretable as a \emph{multi-party competition between explanations}: for
each sample, clusters ``bid'' according to their current code lengths (negative ELBO / negative log-likelihood), and the
winner claims the sample; then each cluster improves its compressor on what it currently owns.

Seen through this lens, the commonality across these efforts is a single meta-problem: \emph{minimize the code length of a partitioned
representation of a collection}. Our contribution is to make the ``degrees of freedom'' explicit by decoupling (i) the
objective (description length / complexity overhead), (ii) the surrogate code family used to upper-bound it, and (iii) the
search protocol (which local moves are allowed and how competing explanations are compared). The cited methods can then be
understood as particular design points (specific codes + specific move sets), while our ``dialectics'' serves as the
umbrella formulation.

\paragraph{Mixture models and EM as a probabilistic dialectics.}
Mixture models provide another canonical realization of ``segments compete to explain data'' \citep{McLachlanPeel2000}.
With EM \citep{Dempster1977}, each component competes for responsibility mass according to likelihood; then parameters are
updated to better explain the assigned data. Topic models such as LDA follow the same theme at the level of discrete latent
structure: topics compete to explain words under a shared generative story \citep{BleiNgJordan2003}.
These methods are typically presented in probabilistic language, but they can also be interpreted as coding schemes (negative
log-likelihood as code length) and hence as instances of competitive MDL.
Our Kolmogorov-style development abstracts away from committing to a particular probabilistic family: in principle, one can
use \emph{any} agreed-upon code family to upper-bound the relevant conditional complexities, and dialectical steps then compare
those bounds.

\paragraph{Image segmentation on pixels.}
Pixel-wise image segmentation can be viewed as partitioning an image into \emph{patches/regions} (foreground/background, object slots, superpixels, etc.), i.e., assigning each pixel (often together with its coordinate) to one of several parts.  In this ``pixels as a collection'' view, a segmentation is a particular \emph{lossless factorization} of the image: given the whole image $I$ and (say) one part $P_2$, the other part $P_1$ is determined as the complementary restriction once the split/boundary is known (and in many practical encodings the boundary is recoverable from either side via pixel indices).  This aligns naturally with our \emph{determination} perspective in the explicit-boundary regime, where $(I,P_1,P_2)$ forms a reversible consistency constraint and the key degree of freedom is \emph{where the boundary is drawn}.

A classical line makes the ``segmentation $\leftrightarrow$ good compression'' connection explicit by combining a \emph{region model} (how to code pixels inside a region) with a \emph{boundary regularizer/encoding} (how to code the interface).  Region Competition \citep{ZhuYuilleRegionCompetition} unifies snakes and region growing under a Bayes/MDL objective: regions are explained by simple appearance/texture models (often Gaussian-like statistics in an appropriate feature space), while the evolving boundary is regularized and optimized via continuous-time dynamics (PDE/curve evolution).  Although written as an energy-minimization procedure, its structure matches our view: two adjacent regions effectively ``bid'' for boundary moves by how much they reduce total description length (data-fit) versus how much boundary complexity they introduce.

A closely related MDL-style formulation appears in Natural Image Segmentation with Adaptive Texture and Boundary Encoding \citep{MobahiNaturalImageSegmentation}, which explicitly counts code length for (i) within-region texture (e.g., a Gaussian model on local-patch features) and (ii) boundary representation (e.g., chain-code variants with adaptive entropy coding).  The optimization is carried out by greedy region merging on a region adjacency graph (often starting from superpixels), i.e., a restricted family of \emph{local moves}.  In our terminology, this corresponds to searching over a constrained move set (merge operations), with a concrete, computable surrogate for the Kolmogorov objective.

More recent object-centric segmentation work can be interpreted as choosing a different surrogate for conditional description length.  Binding via Reconstruction Clustering \citep{BindingViaReconstructionClustering} tackles the ``binding'' problem (assigning pixels/features to object slots) by iterating between (soft/hard) assignments and slot-wise reconstruction; the implicit ``code'' is a reconstruction loss, and each slot competes to explain pixels by improving its ability to reconstruct them.  Similarly, Contextual Information Separation methods \citep{YangContextualInformationSeparation,SavareseInpaintingErrorMaximization} learn masks by coupling a segmenter with a predictor/inpainter: the segmenter proposes a partition, and the predictor tries to infer one region from the other.  This is often implemented as a GAN-like or adversarial objective where the segmenter aims to \emph{maximize} cross-predictability error (make regions mutually unpredictable) while keeping each region individually predictable under its own model.  Our formulation is close in spirit of reducing mutual information but differs in emphasis: rather than a segmenter-predictor game, we make the \emph{regions themselves} the players in a dialectical game, each trying to absorb contested pixels so as to minimize its \emph{own} conditional description length (an operational surrogate for $K(X\mid P_i)$), with feasibility enforced by the determination constraints.

Taken together, these pixel-level segmentation literatures can be read as different instantiations of the same core: choose a partition/boundary so that the induced parts admit short codes under some family (Gaussian texture models, boundary codes, reconstruction-based models, infilling predictors).  Dialectics shows that image segmentation can be unified with clustering methods by (i) treating reversibility/determination as the shared structural constraint and (ii) treating ``who gets this pixel/element/patch?'' as a competitive coding decision, while allowing the practitioner to swap in different effective compressors and different local-move protocols (PDE evolution, merging, slot competition, adversarial infilling) under one modality-agnostic complexity objective.

\paragraph{Disentangled representation learning and the need for inductive bias.}
The disentanglement literature has made a sobering point precise: without additional inductive
biases, the ``ground-truth'' factorization is in general \emph{not identifiable} from $p(x)$ alone \citep{LocatelloEtAl2019}. We view this not as a defeat of unsupervised concept discovery, but as a pointer to a missing ingredient:
certain objectives treat observationally equivalent explanations as indistinguishable, because they ignore the
\emph{description/implementation complexity} of the explanation itself.
Once we treat low-Kolmogorov-complexity as a universal inductive bias, the non-identifiability result coexists naturally with our
program: among many representations that realize the same $p(x)$, we prefer those admitting shorter effective descriptions.

Concretely, with many observations $X=(x_1,\dots,x_N)$, we seek two latent sequences
$Z^{(1)}=(z^{(1)}_1,\dots,z^{(1)}_N)$ and $Z^{(2)}=(z^{(2)}_1,\dots,z^{(2)}_N)$ such that they form a determination
\emph{elementwise}:
\[
\forall i\in[N],\qquad
K(x_i\mid z^{(1)}_i,z^{(2)}_i)\stackrel{+}{=}0,\quad
K(z^{(1)}_i\mid x_i,z^{(2)}_i)\stackrel{+}{=}0,\quad
K(z^{(2)}_i\mid x_i,z^{(1)}_i)\stackrel{+}{=}0.
\]
In this view, a ``disentangled'' representation corresponds to finding such multi-part determinations where the parts remain
\emph{non-redundant} under an effective coding scheme (low excess information), so that the two sides behave as distinct
concepts rather than an arbitrary entangled reparameterization.
Dialectical optimization is our proposed search protocol for this goal: competing parts bid for explaining contested
information via conditional code lengths, and we retain reallocations that decrease the tracked description length.

Using Kolmogorov complexity as a universal inductive bias is of course not new. It is central to Solomonoff induction
(e.g., \citealp{Solomonoff1964}), but here we emphasize the distinctive structure it induces in \emph{concept discovery}:
the bias manifests operationally through competitive coding (who can explain a contested piece more cheaply), which yields stable, communicable low-complexity
concepts.

\paragraph{LLM prompting for ``dialectical'' reasoning and our distinction.}
A separate, more recent thread uses ``dialectics'' as a prompting or evaluation motif for LLM reasoning, e.g., by eliciting
self-reflection or multi-agent debate and then scoring reasoning chains \citep{SelfReflectingLLM2025,MultiModelDialecticalEval2025}.
We view these as valuable engineering instantiations of debate-style scaffolding, but conceptually orthogonal to our main
claim: dialectics is not merely a prompt pattern; it is the \emph{computational structure} that emerges when multiple
explanations (concepts) compete under a shared description-length objective and constrained reversibility. In particular, our
``players'' are not necessarily LLM agents; they can be regions, clusters, latent components, or any representational parts
that bid for information by compressing it.

\paragraph{Adversarial learning and predictability minimization.}
Our work is thematically a special case of broader adversarial learning paradigms.
GANs frame learning as a two-player game between a generator and a discriminator \citep{GoodfellowEtAl2014}, and earlier
predictability minimization work explicitly trains representations by making certain variables unpredictable from others
\citep{Schmidhuber1992}. These methods optimize objectives that are game-structured, but typically with \emph{opposed} goals
(fool vs.\ detect, predict vs.\ prevent prediction).
Dialectics differs in that both sides are optimizing \emph{the same type of quantity}, that is description length, under a shared
feasibility constraint; the ``opposition'' comes from competition over \emph{allocation} (which concept gets to explain which
piece), not from asymmetric loss functions.

Finally, as our terminology already suggests, this work is consciously inspired by Hegelian dialectics.

\paragraph{Summary and positioning.}
Across these literatures we repeatedly see the same structural pattern: a system proposes a \emph{partition/assignment} of information (points to clusters, pixels to regions, observations to latent components), and then \emph{adjudicates} among competing explanations by some effective code-length surrogate (likelihood, reconstruction loss, MDL/MML, prequential log-loss), iterating local updates until a stable articulation emerges.

We use \emph{AI dialectics} as the most general description of this pattern: an optimization problem whose objective is (a computable upper bound on) a \emph{Kolmogorov-style description length}---hence modality-agnostic---and whose defining feature is a \emph{shared} description-length objective together with an global or local competition protocol among feasible solutions.

Importantly, \emph{AI dialectics} itself is a \emph{concept}: it groups a family of works that instantiate the same competitive-coding core under different code families (Gaussian/MDL, deep generative models, reconstruction-based surrogates), different constraints (e.g., geometric boundaries vs.\ collection structure vs.\ sequence structure), different modalities, and different implementations. This usage also separates ``dialectical'' mechanisms (explicit competition over explanation territory) from abundant other organizing principles in the literature (some examples are supervised learning, standard ELBO optimization). Therefore, at a meta-level, \emph{dialectics and non-dialectics together with the whole literature also form a determination} in our sense. \section{Some Practical Methods in AI Dialectics}
\label{sec:practical-methods}

This section shows a few concrete, implementable procedures of AI dialectics for \emph{text}. Compared to vision/audio, the text domain has an unusually convenient primitive: modern language models already expose a
high-quality, computable surrogate for description length via token log-probabilities.
Our goal here is not to propose one single canonical algorithm, but to provide a small toolbox of design patterns that
instantiate the same core idea:

\begin{quote}
\emph{multiple concepts compete to explain the same piece of information by compressing it, we keep the allocation that yields shorter code.}
\end{quote}

We focus on three increasingly ``active'' instantiations: (i) \emph{logprob clustering} with a frozen model,
(ii) \emph{LoRA clustering} \citep{HuEtAl2022LoRA} where each concept learns a better compressor for what it owns, and
(iii) \emph{competitive encoding for segmentation} inside a long sequence.

\subsection{A Computable Surrogate: LLM Log-Probabilities as Code Lengths}
\label{sec:practical-logprob}

Kolmogorov complexity is uncomputable, but dialectics only needs a \emph{reproducible upper bound} that can be compared across competing explanations.
For text, an autoregressive language model $M$ provides exactly such a bound.

Let $x=(x_1,\dots,x_T)$ be a token sequence and let $c$ be a context string (prompt).
If the API returns token log-probabilities $\log p_M(x_t\mid c,x_{<t})$ (in natural log or base-10),
we define the \emph{code length} (in bits) as
\begin{equation}
\label{eq:llm-codelen}
L_M(x\mid c)\;:=\;-\sum_{t=1}^T \log_2 p_M(x_t\mid c,x_{<t}).
\end{equation}
This quantity is not merely a heuristic score: under a standard entropy-coding view,
an autoregressive model $M$ induces a concrete \emph{sequential code} for $x$ given $c$
(e.g.\ arithmetic/range coding driven by the conditional probabilities $p_M(\cdot\mid c,x_{<t})$).
Such a coder produces a bitstring whose length is approximately \eqref{eq:llm-codelen},
and therefore yields a \emph{computable upper bound} on Kolmogorov complexity relative to a fixed decoding protocol:
for an appropriate reference machine $U$ that implements the agreed tokenizer/model/coder,
\[
K_U(x\mid c)\;\le\;L_M(x\mid c)+O(1).
\]
Crucially, the language model, tokenizer conventions, separators/formatting rules, truncation, and the entropy coder are all part of this public protocol (hence part of the effective context).

\subsection{Logprob Clustering: Competitive Coding on a Collection}
\label{sec:practical-logprob-cluster}

\paragraph{Problem.}
Given a collection of strings $S=\{s_1,\dots,s_N\}$ and an integer $K$, produce a partition into $K$ clusters
$C_1,\dots,C_K$ such that items grouped together can \emph{explain/compress} each other well.

\paragraph{Dialectical viewpoint.}
Each cluster $C_i$ acts as a \emph{concept} that can spend computation to compress a candidate string $s$.
The cluster that achieves the shortest conditional description length for $s$ wins $s$.

\paragraph{A minimal assignment rule.}
Fix a deterministic function $\mathrm{ctx}(\cdot)$ that turns a cluster (a set of strings) into a context prompt
(e.g., concatenate a small representative subset with separators).
Given a candidate $s$, score each cluster $i$ by the conditional code length $L_M\!\bigl(s \mid \mathrm{ctx}(C_i)\bigr)$,
and assign $s$ to the cluster achieving the smallest value (shorter code wins).
If $s\in C_i$, score it \emph{leave-one-out} to avoid trivial self-explanation:
use $\mathrm{ctx}(C_i\setminus\{s\})$.

\paragraph{Algorithm.}
\begin{enumerate}
\item \textbf{Initialize.} Randomly assign each $s\in S$ to one of $K$ clusters (or use any heuristic seed).
\item \textbf{Repeat until stable (or for $T$ rounds):}
      \begin{enumerate}
      \item For each cluster $i$, build a context $\mathrm{ctx}(C_i)$ (deterministically).
      \item For each string $s\in S$, compute $L_M\!\bigl(s \mid \mathrm{ctx}(C_i)\bigr)$ for all $i\in[K]$
            (leave-one-out when $s\in C_i$), and reassign $s$ to the cluster with the smallest code length.
      \end{enumerate}
\end{enumerate}

This is the simplest ``multi-party debate'' you can run with an LLM logprob API:
clusters compete by conditional compression; the partition is the stable outcome.

\paragraph{Practical details that matter.}
\begin{itemize}
\item \textbf{Context budgeting.} Because prompts have a context window, $\mathrm{ctx}(C_i)$ should be a bounded summary.
      A simple choice is: sample a fixed number $m$ of representatives from $C_i$, concatenate with separators.
\item \textbf{Caching.} Many scores repeat across iterations; cache $(\mathrm{ctx}(C_i),s)\mapsto L_M(s\mid \mathrm{ctx}(C_i))$.
\item \textbf{Margin for stability.} Only move $s$ if the best code length improves on the current cluster's code length by a small margin,
      to avoid oscillations due to estimation noise.
\end{itemize}

\subsection{LoRA Clustering: Letting Concepts Learn Their Own Compressors}
\label{sec:practical-lora-cluster}

Logprob clustering treats the model as a fixed compressor.
A natural next step is to let each cluster \emph{improve its compressor} on what it currently owns, while still using the same
dialectical adjudication rule (shorter code wins).

\paragraph{Setup.}
Fix a base language model $M_0$ and attach $K$ lightweight adapters (e.g.\ LoRA) producing $K$ conditional models
$M_1,\dots,M_K$.
Each $M_i$ is a candidate ``conceptual compressor''.

\paragraph{EM-style dialectics.}
\begin{enumerate}
\item \textbf{E-step (competition / assignment).}
For each string $s$, compute its code length under each cluster model:
\[
\ell_i(s)\;:=\;L_{M_i}(s\mid c_i),
\]
where $c_i$ is either empty or a small cluster-specific context (optional).
Assign $s$ to the cluster with smallest $\ell_i(s)$.
\item \textbf{M-step (improve compressors).}
For each cluster $i$, update its adapter parameters by training on the strings currently assigned to it,
minimizing the same loss (negative log-likelihood), possibly with regularization.
\end{enumerate}

\paragraph{Why this fits the dialectical reading.}
The ``players'' (clusters) are not merely labels, they actively \emph{learn better arguments} (shorter codes) for the data they
own.
Competition happens in the E-step, assimilation happens in the M-step.

\paragraph{Collapse as \emph{Aufheben} (and a practical merge-and-restart move).}
In practice, the ``collapse'' phenomenon, that is one adapter wins most assignments and the other adapters become irrelevant, is not merely a training pathology, but an important phenomenon in its own right.
We call such events an \emph{Aufhebung} (sublation): a plurality of competing concepts is lifted into a new unity.
A general theoretical account of Aufhebung is beyond the scope of this paper, but there is a simple operational move that makes collapse useful.

\emph{Merge the winner.} Take the winning adapter (e.g.\ LoRA) and merge its weight update into the base model, producing a new base $M_0'$.
Empirically, $M_0'$ can be a better global compressor for the full dataset than the original $M_0$.
Then restart the procedure from $M_0'$ with freshly initialized adapters, so that the next round of competition factors residual structure rather than re-learning the same global structure that has already been absorbed.

\subsection{Competitive Encoding for Text Segmentation}
\label{sec:practical-competitive-encoding}

So far we treated a piece of text as a single unit.
Here we instead consider its internal division. For a single long sequence (a document, transcript, or code file), the natural ``segmentation'' variable is a boundary.

\paragraph{Motivation: a boundary as ``which compressor explains which span''.}
A long sequence often contains internally heterogeneous structure: topic shifts, speaker changes, style changes, or a transition from narrative to technical details.
A principled way to discover such structure is to let \emph{multiple compressors compete} for different spans and choose a boundary that minimizes the resulting total code length.
This idea has appeared in earlier forms \citep{BaxterOliver1996KindestCut,OliverBaxterWallace1998MMLSegmentation,BohlinEtAl2014MapEquation},
but our setup is different: we emphasize \emph{autoregressive} coding of a \emph{single} sequence and treat segmentation explicitly as a \emph{dialectical boundary optimization}
between two anchored explanation regimes.

\paragraph{A simplified setup.}
We start from the simplest nontrivial case: a single breakpoint that splits the sequence into two contiguous parts.

Let the text be split into atomic blocks
\[
S \;=\; X_1 X_2 \cdots X_T,
\]
where each $X_t$ is, for example, a sentence, a line, a paragraph, a turn in dialogue, or a token.

We assume two \emph{grounds} $\hat A,\hat B$ which pin down two coding regimes.
Concretely, these grounds can be implemented as:
(i) two prompts for the same base LLM, (ii) two adapters (e.g.\ LoRA heads) attached to a shared base model, or
(iii) two separate probabilistic models.
In all cases, the only requirement is that encoder and decoder share the same public protocol for how $\hat A$ or $\hat B$
selects a predictive distribution over next tokens.

\paragraph{A breakpoint as a determination.}
For any $\tau\in\{0,1,\dots,T\}$, the triple $(S, X_{1:\tau}, X_{\tau+1:T})$ is one determination node inside a larger grounded diagram.
\Cref{fig:competitive-segmentation-two-regimes} makes explicit (i) dependence on the two grounds $\hat A,\hat B$ that pin the two coding regimes, and (ii) the sequential dependence of later blocks on earlier blocks across the boundary (the right span is evaluated conditional on the decoded left span).

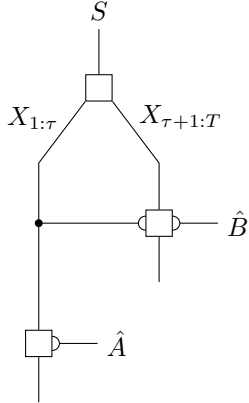
\begin{figure}[h]
\centering
\begin{tikzpicture}
    \node[box] (Top) at (0, 2.2) {};
    \node[box] (MidR) at (0.8, 0.4) {};
    \node[box] (BotL) at (-0.8, -1.2) {};

    \draw (Top.north) -- ++(0, 0.6) node[above] {$S$};

    \draw (Top.south west) -- (-0.8, 1.2) coordinate (elbowA) -- (BotL.north);
    \draw (BotL.south) -- ++(0, -0.6);
    \draw (MidR.south) -- ++(0, -0.6);
    \path (Top.south west) -- (elbowA) node[midway, above left, inner sep=1pt] {$X_{1:\tau}$};

    \draw (Top.south east) -- (0.8, 1.2) coordinate (elbowB) -- (MidR.north);
    \path (Top.south east) -- (elbowB) node[midway, above right, inner sep=1pt] {$X_{\tau+1:T}$};

    \draw[arc] (MidR.west) -- (-0.8, 0.4);
    \fill (-0.8, 0.4) circle (1.5pt);

    \draw[arc] (MidR.east) -- ++(0.6, 0) node[right] {$\hat{B}$};
    \draw[arc] (BotL.east) -- ++(0.6, 0) node[right] {$\hat{A}$};
\end{tikzpicture}
\caption{Competitive encoding with a single boundary: two grounded coding regimes $\hat A,\hat B$ compete to explain different contiguous spans of the same sequence $S$.}
\label{fig:competitive-segmentation-two-regimes}
\end{figure}

\paragraph{A computable proxy.}
We use grounded autoregressive code lengths $L_{\hat A}(\cdot)$ and $L_{\hat B}(\cdot)$ to define
\begin{equation}
\label{eq:two-regime-codelen}
\mathcal{L}(\tau)
\;:=\;
L_{\hat A}\!\bigl(X_{1:\tau}\bigr)
\;+\;
L_{\hat B}\!\bigl(X_{\tau+1:T}\mid X_{1:\tau}\bigr)
\end{equation}
and take
\begin{equation}
\label{eq:tau-star}
\tau^\star \;\in\;\arg\min_{\tau\in\{0,\dots,T\}} \mathcal{L}(\tau).
\end{equation}

The same template can be extended to multiple boundaries (change-point coding) and to learned adapters/heads (e.g.\ LoRA).
 \section{Speculative Future Directions}
\label{sec:next-steps}

So far we framed \emph{AI dialectics} as an optimization protocol for finding low-excess determinations and
stable boundaries via competitive coding.
In this section we sketch a few \emph{speculative} directions suggested by that framing.

\subsection{Features as Computed Side Information for Compression}
\label{sec:features-as-sideinfo}

One of the central open questions in deep learning is \emph{why feature learning works so well}.
Empirically, in modern \emph{unsupervised} pipelines, we train a powerful compressor/predictor on a dataset $X$
(e.g.\ by maximizing likelihood / minimizing prequential log-loss), and \emph{useful features} often \emph{emerge} as a byproduct:
the learned internal representations transfer to many downstream tasks.

Here we consider a reverse viewpoint:
Instead of hoping that good compression \emph{incidentally} yields good features,
can we \emph{design} features (as computable functions) whose only purpose is to \emph{improve compression}?

\paragraph{The reverse problem: synthesize a feature and then compress.}
Let $x$ be the dataset that need to be compressed, and let $f$ be a computable \emph{feature function}.
We generate auxiliary information
\[
y \;=\; f(x),
\]
and then compress the \emph{augmented} data $(x,y)$ rather than $x$ alone.
At first glance this seems wasteful and would lead to longer code length: encoding $y$ should increase code length because $y$ is additional data.

However, at the level of ideal Kolmogorov complexity, adding a \emph{deterministic} feature does not change
the information content \emph{given} the feature program.
Indeed, since $y$ is computable from $(x,f)$, and $x$ is trivially computable from $(x,y,f)$, we have
\begin{equation}
K(x \mid f)\ \stackrel{+}{=}\ K(x,y \mid f)\big|_{y=f(x)}.
\label{eq:feature-does-not-add-ideal-info}
\end{equation}
So this construction does \emph{not} beat information conservation: adding $y$ does not increase or reduce the code length for compressing $x$.

\paragraph{Why it can still help: practical compressors are not universal.}
The key is that real-world compressors are restricted code families.
Fix a lossless compressor (or coding protocol) $\mathsf{C}$, and write $\ell_{\mathsf{C}}(x,y)$ for the achieved code length
when encoding the pair $(x,y)$.\footnote{Concretely, $\mathsf{C}$ may be an entropy coder with a chosen probabilistic model,
a prequential predictor, an MDL two-part code, or any agreed-upon description method that is effectively decodable.}
Then $(x,y)$ is an \emph{effective} description of $x$ given $f$ and $\mathsf{C}$ (decode $(x,y)$, discard $y$), hence
\begin{equation}
K(x\mid f,\mathsf{C})\ \stackrel{+}{\le}\ \ell_{\mathsf{C}}(x,y)\big|_{y=f(x)}.
\label{eq:feature-upper-bound}
\end{equation}
Equation \eqref{eq:feature-upper-bound} turns feature design into a principled compression problem:
\emph{choose $f$ so that a fixed practical compressor can describe $(x,f(x))$ with a short code.}
Even though $y$ must also be encoded, a well-chosen $y$ can expose structure that $\mathsf{C}$ fails to exploit from $x$ alone,
making the joint description shorter (and sometimes even shorter than $\ell_{\mathsf{C}}(x)$ under the same protocol).

Therefore, the corresponding objective for finding good feature is
\[
\min_{f}\ \ \ell_{\mathsf{C}}(x, f(x)).
\]

\paragraph{Where dialectics enters: low-complexity determinations as feature generators.}
Our hypothesis is that \emph{dialectics} is an effective search protocol for discovering such $f$'s.
Informally, a low-complexity determination, when instantiated on data, can be materialized as a computable mapping that produces auxiliary variables
(boundaries, assignments, dictionaries, or other structured side information) that are:
\begin{enumerate}
\item \textbf{cheap to specify} (requires little or no human annotation),
\item \textbf{scalable across the dataset} (from sparse seeds, the concept can ``grow'' to cover the dataset),
\item \textbf{compression-relevant} (built by optimizing compression, yields determinations with low excess information).
\end{enumerate}
In this sense, a determination is not merely a \emph{segmentation output}, it is a candidate \emph{feature mechanism} $f$.

\paragraph{A falsifiable evaluation protocol.}
We sharpen the above hypothesis into a concrete, falsifiable protocol by constraining ourselves to a deliberately
\emph{weak} compressor $\mathsf{C}_0$ (a limited code family / model class) and an iterative dialectical procedure:
\begin{enumerate}
\item use $\mathsf{C}_0$ to score dialectical moves on raw data and produce an initial feature/side-information $y_1=f_1(x)$,
\item re-run the same weak-compressor scoring on the augmented object $(x,y_1)$ to produce a higher-level feature $y_2=f_2(x,y_1)$,
\item iterate, accepting only features that reduce a tracked code length $\ell_{\mathsf{C}_0}(x,y_{\le t})$.
\end{enumerate}
The hypothesis is that for some modest number of rounds $T$, the resulting representation $y_{\le T}$, with only shallow
downstream learning, can match (or approach) the performance of current pipelines that rely on non-dialectical feature emergence.
If true, this would upgrade ``compression as intelligence'' from a passive observation (features emerge) into an
active design principle (features are synthesized to improve compression). Future work will determine whether this program
succeeds in practice, and under which code families and move sets it becomes competitive.

\subsection{Scale Compute for Dialectics}
\label{sec:dialectics-compute}

The current dominant paradigm for scaling capability is to spend compute on gradient-based optimization of large parametric models.
Dialectics suggests a complementary axis: spend compute on \emph{searching for low-excess determinations}.

Three properties make dialectical search attractive as a scaling target:
\begin{enumerate}
\item \textbf{Parallelism by design.} Competitions are local (a patch $X$ is contested by a few concepts),
      so many contests can run in parallel across a dataset, an image, or a document collection.
\item \textbf{Cacheability and reuse.} Once a concept has paid the compute cost to compress a certain kind of patch,
      that ``argument'' (a learned code/model or a cached sufficient statistic) can be reused across future patches.
\item \textbf{A native communication-computation trade-off.} Concepts can be transmitted as small grounds/seeds and reconstructed by running dialectical growth.
      With sparse grounds, more compute is needed to rebuild a stable boundary from scratch. With richer grounds, less compute is needed, and in the limit one could transmit the full boundary and only spend compute to verify it under the shared protocol.
\end{enumerate}

\subsection{AI Dialectics as a Bridge between Neural Models and Symbols}
Neural models provide a computable proxy for description length (e.g., log-probabilities or reconstruction codes), and dialectical competition crystallizes these numbers into discrete and reusable objects, like boundaries and clusters, that form a symbolic layer. This layer can then be lifted into an explicit symbol system (e.g., predicates, rules, graphs), enabling neural-symbolic coordination between neural network learning and discrete reasoning.

\subsection{Cross-Modality Ontologies for World Models}
A ``world model'' in our framing is not a single monolithic predictor, but an evolving \emph{ontology of concepts} that consistently links multiple modalities. Concretely, suppose an underlying event or situation generates a bundle of observations
\[
o = (o^{\text{img}}, o^{\text{text}}, o^{\text{audio}}, \dots).
\]
Each modality offers its own effective coding scheme (a vision codec, an ASR model, an LLM, \dots), hence its own computable surrogate for description length. The cross-modality problem is then: \emph{find concepts that are simultaneously compressive and mutually identifiable across these views}.

Determinations are very suitable here because the dialectical/compression view does not privilege any modality in the first place. We treat every observation bundle, which could include images, text, audio, trajectories, tables, etc., as a single information object $I$ (a binary string). A ``segmentation'' is then just applying some split operators to $I$ to produce multiple parts:
spatial cropping of an image, temporal windowing of a stream, cutting a text sequence into spans, partitioning a dataset into subsets, or even additive decompositions of vectors (e.g.\ separating an audio mixture into components). These operators can be freely composed and mixed: one may first slice by time, then carve regions in each frame, then select subsets of samples, and so on.

A (cross-modal) concept is simply one of such parts $(P_1,\dots,P_m)$ that connects to one determination node together with $I$. One concept could simultaneously live in different ``modalities'', if only it satisfies recovery constraint imposed by determination, i.e.\ $\forall i,\ K(P_i\mid P_{-i}, I)\stackrel{+}{=}0$ and $K(I\mid P_1,\dots,P_m)\stackrel{+}{=}0$.
For example, let $I$ be the entire Internet: one possible split is $P_1$ collecting all cat-related information (cat regions in images, cat clips in videos, textual descriptions of cats, and cat meows in audio), and $P_2$ collecting everything that is not about cats.

Dialectics is then the optimization mechanism that discovers and stabilizes such low-excess configurations by letting candidate parts compete to compress contested patches, without ever needing to treat ``modality'' as a special primitive.
In particular, a concept can ``expand'' by mining cross-modal predictive links, 
e.g.\ textual descriptions of cat habits or breeds can be aligned with visual cat features, 
and any coding scheme that can exploit such links will help concept of ``cat'' to naturally extend from text only into a cross-modal one.

\subsection{A Common Ontology Server as Shared Infrastructure}
\label{sec:common-ontology-server}

In order to make dialectics function beyond a single intelligent being setting, we need to tackle the problem of communicating ``the same concept'' across independent participants.

Consider a set of intelligent participants:
humans, model instances, multi-agent systems, organizations, or other computational actors.
Assume each participant has sufficient compute to run a dialectical protocol, but no privileged trust relationship
with others. The key friction is not only that participants may disagree about the \emph{claims} about a concept,
but that they may fail to even \emph{point to the same concept}.

To this end, we need shared
infrastructure that makes \emph{concepts} communicable, comparable, and reproducible. This subsection proposes such an infrastructure: a \emph{common ontology server}.
The server does not ``decide truth'', rather, it provides a public,
verifiable substrate for (i) naming determinations, (ii) reproducing them under a shared protocol,
and (iii) challenging and improving them via dialectical competition.
In short: it is a registry and execution layer for \emph{concepts as computed objects}.

\paragraph{Core operations: read, create, reproduce, challenge.}
The server exposes a minimal set of operations that are intentionally symmetric across participants.

\begin{enumerate}
\item \textbf{Read.} Fetch a concept record:
its determination specification, its declared objectives, and any public evidence (grounds, cached scores, provenance log).
\item \textbf{Create.} Register a new candidate concept:
upload its specification and a minimal ground set (e.g.\ a few typical samples, seeds, or constraints) sufficient to reconstruct it.
\item \textbf{Reproduce.} Given an identifier, rerun the dialectical protocol locally or via a server-executed job
to reconstruct the concept, and verify the stated contracts.
\item \textbf{Challenge.} Propose a dialectical move that leads to alternative determination;
the protocol adjudicates the challenge by re-scoring under the shared $\mathsf{C}$ and verification rules.
If the challenger improves the score, the concept is updated.
\end{enumerate}

Crucially, the server's role is not to ``declare winners by authority'' but to \emph{anchor reproducibility}.
Any participant can verify outcomes by re-executing the same protocol on the same inputs.

\paragraph{Verifiability: what can be checked and what cannot.}
Two kinds of claims appear in this framework:

\begin{itemize}
\item \textbf{Correctness of a code.} A claimed code length under a specific decodable protocol $\mathsf{C}$ is verifiable:
one can decode and check equality with the target object (or check a prescribed reconstruction error tolerance).
This makes it hard to cheat about \emph{achieved} description length, as long as the reference machine and allowed side-information are fixed.
\item \textbf{Optimality of a concept.} Claims of global optimality are not verifiable in general
(because Kolmogorov-optimality is only upper semi-computable).
The server therefore treats every concept as an \emph{upper bound} and every improvement as a \emph{better upper bound}.
Stability is empirical and procedural: a concept is ``current best'' if it has survived challenges and no known move improves it.
\end{itemize}

This is an important philosophical and engineering boundary:
the ontology server supplies a \emph{public game board} and a \emph{comparable score},
not an oracle for metaphysical certainty.

\paragraph{Preventing ``hidden information'': standardizing the reference machine.}
The main way to ``cheat'' compression-based evaluation is to smuggle information through an unshared library,
a proprietary model checkpoint, or an implicit assumption about the interpreter.
To keep comparisons meaningful, the server must define a \emph{strict} evaluation environment:

\begin{itemize}
\item a prefix-free universal machine $U$ (or a small set of sanctioned machines),
\item a versioned standard library $\mathcal{L}$ (including any foundation models if they are allowed at all \citep{BommasaniEtAl2021FoundationModels}),
\item explicit accounting of all auxiliary resources (model parameters, retrieval indices, external databases),
\item deterministic build rules (so that ``the same concept id'' implies the same executable semantics).
\end{itemize}

Under this discipline, description length becomes ``uncheatable'' in the limited but practically crucial sense that:
\emph{any claimed improvement must be realizable as an executable, decodable artifact under the shared environment.}
Participants can still disagree on whether the environment is the ``right'' one, but the server forces that disagreement to be explicit.

\paragraph{Caching as shared compute: the server as a concept compute substrate.}
Dialectics can be compute-hungry: repeated contests over patches, repeated scoring under $\mathsf{C}$,
and repeated reconstruction checks.
A core infrastructural advantage of a server is that it can host \emph{reusable caches}:
conditional code-length estimates, learned local models, sufficient statistics, partial segmentations, or verified intermediate artifacts.

We emphasize two design constraints:
\begin{enumerate}
\item \textbf{Caches must be optional.} A participant should be able to reproduce a concept from first principles
using only the specification and grounds, and caches merely accelerate.
\item \textbf{Caches must be auditable.} The server should attach verification metadata:
hashes of inputs, deterministic seeds, and a replay recipe so that caches can be recomputed or spot-checked.
\end{enumerate}

\paragraph{Dialectics as a public adversarial process.}
A common ontology server is not only a registry; it is a \emph{public dialectical arena}.
Any participant can submit counterexamples, propose boundary flips, or introduce alternative determinations.
This creates a built-in adversarial pressure that is aligned with the protocol:
unstable or spurious concepts will be challenged and either refined or discarded.

In particular, the server supports a procedural notion of \emph{robustness}:
a determination is robust if its compressive advantage persists under (i) new data, (ii) adversarially selected patches.
This is the infrastructural analogue of scientific replication:
concepts are not accepted because an authority endorses them, but because they remain compressive under repeated attacks.

\paragraph{Scope and limitations.}
The proposed infrastructure is deliberately limited in what it claims.

\begin{itemize}
\item It does \emph{not} solve the uncomputability of $K(\cdot)$, and it does not certify global optimality.
\item It does \emph{not} eliminate pluralism: different communities may adopt different $\mathsf{C}$ or different reference machines.
Instead it makes those choices explicit and versioned.
\item It raises real deployment concerns (privacy of $I$, governance of the reference environment, concentration of compute),
which must be addressed if such a server is to become legitimate shared infrastructure.
\end{itemize}

Nevertheless, if successful, a common ontology server would operationalize the slogan
\emph{``dialectics is a concept engine''}:
it turns compute into reproducible concepts, makes concepts addressable and debatable, and reduces the cost for
distributed agents to coordinate on ``what they are talking about''.
This could serve as a missing infrastructural layer between modern neural compressors
(which provide scalable, computable proxies for description length)
and the symbolic objects (boundaries, clusters) that enable stable reasoning and communication.

\subsection{Dialectics at World Scale}
\label{sec:dialectics-world-scale}

Our core premise here is that \emph{concepts are built on experience}.
In particular, a concept is not an arbitrary isolated string: it only ``counts'' as a concept if it is
\emph{tied to the total experience} $I$ by determination constraints, for instance, when there exists some complementary remainder $R$ such that
$(I,C,R)$ forms a $3$-way determination (ideally with low excess).
This enforces a strong structural restriction: concepts must connect to the whole, because recoverability is defined relative
to the whole.

\paragraph{Richer experience enables broader, finer concepts.}
As experience expands, the global object $I$ contains more repeated structure, longer-range regularities, and cross-context invariants.
Only then do certain \emph{macro} concepts become viable: they are not supported by any single patch, but by many dispersed instances whose common pattern becomes visible at scale.
In short, richer $I$ can support a concept system with wider scope and higher-level abstractions, because those abstractions become reusable compressive handles across the enlarged experience.

For example, a child's experience is local and sparse, so their stable concepts are mostly concrete: ``toy'', ``home'', ``dog'', ``hungry''.
As experience accumulates through school, work, relationships, and repeated exposure to social and economic situations, more complex macro concepts begin to emerge: ``fairness'', ``reputation'', ``career'', ``inflation'', ``social norms''.
None of these live in a single episode, they crystallize only when many episodes can be compressed and navigated more effectively by reusing the same high-level handle.

\paragraph{How ``world scale'' changes the picture.}
At Internet scale the experience budget dwarfs any individual lifetime.
As a crude accounting example, a corpus of $T\approx 10^{13}$ tokens has $\log_2 T \approx 43$ bits of positional
address information per token (just to specify an index).
Averaged over $8\times 10^9$ people, this corresponds to only about $10^3$ tokens per person, roughly the amount of text on a single page of a typical conference paper.

Many existing concept-mining methods already work well at roughly this \emph{local} scale: given a single short document (or a small context window),
they can often answer ``what is the key concept in this span?'' or ``which feature explains these words/pixels?''.
In computer vision, promptable segmentation foundation models such as Segment Anything (SAM) and its variants make it easy to propose boundaries around an object in one image \citep{KirillovEtAl2023SegmentAnything,ZhangEtAl2023MobileSAM,KeEtAl2023SegmentAnythingHQ}.
In interpretability, circuit-style analyses can sometimes isolate a mechanism that explains a narrow behavior of a model on a specific input.

However, at Internet scale, there are cross-domain determinations that are too diffuse for local methods to notice:
patterns that only become compressible when one can connect distant instances (across subcultures, disciplines, languages, and modalities).
It is therefore plausible that an Internet-scale foundation model, trained as a massive compressor of $I$, has already
internalized such cross-context concepts, simply because reusing them reduces description length.

The awkward part is that these concepts need not align with any familiar human category.
A model may exploit a latent concepts that has no widely agreed name, definition, or social convention.
In other words, the system can \emph{use} a concept without humans being able to \emph{point to} it, communicate it, or audit it.

This motivates our proposal: a world-scale dialectical optimizer that \emph{surfaces} previously unnamed concepts and \emph{produces} compact seeds/grounds that pin them down and make them re-identifiable across contexts, enabling human interpretation and audit.
In AI dialectics, ``new concepts'' are not mystical: they are simply \emph{new low-excess determinations} made possible by a large $I$
and enough compute to search for and stabilize them. 
\bibliography{content/references}
\bibliographystyle{plainnat}

\appendix
\section{Proofs}\label{sec:proofs}

\begin{proof}[Proof of \Cref{prop:one-var-completion}]
Let
\[
l \;:=\; \max_{i\in[n]} K(x_i \mid x_{-i}).
\]
We construct a string $x_{n+1}$ (depending on $x_{1:n}$ and the integer $l$) that satisfies \eqref{eq:one-var-feasibility} and has length no larger than $l$, up to logarithmic slack, so also satisfies \eqref{eq:one-var-ub}.  Since $l$ is an integer of magnitude at most the relevant complexities, encoding $l$ costs only $K(l)=O(\log l)$ bits, which is absorbed by our $\stackrel{+}{=}$ notation.

\paragraph{Step 1: An enumerable family of ``$l$-consistent'' tuples.}
Fix $l$. For each index $i\in[n]$, consider the following effective enumeration procedure $\mathsf{Enum}_i(l)$:

\begin{itemize}
\item Enumerate all prefix-free programs $p$ with $|p|\le l$.
\item Dovetail over all $(n\!-\!1)$-tuples $t\in(\{0,1\}^\ast)^{n-1}$.
\item Run $U(p,t)$; whenever it halts with output $y$, output the \emph{candidate} $n$-tuple
\[
\mathsf{ins}_i(t,y)\;:=\;(t_1,\dots,t_{i-1},y,t_i,\dots,t_{n-1}),
\]
together with the tag $i$.
\end{itemize}

Intuitively, $\mathsf{Enum}_i(l)$ enumerates all $n$-tuples $z_{1:n}$ for which the $i$-th coordinate $z_i$ is computable from the other coordinates $z_{-i}$ by some program of length at most $l$.

Now run all $\mathsf{Enum}_i(l)$ for $i=1,\dots,n$ in parallel, and maintain a table that, for every $n$-tuple $z_{1:n}$ ever produced by any $\mathsf{Enum}_i(l)$, stores the set of indices
\[
S(z)\;\subseteq\;[n]
\quad\text{for which $z$ has been produced by $\mathsf{Enum}_i(l)$.}
\]
Whenever a tuple $z$ first satisfies $S(z)=[n]$ (i.e.\ it has been seen from \emph{every} position $i$), we declare it \emph{$l$-consistent} and output it exactly once.

Let $\mathcal{E}_l$ denote the resulting enumerable set of $l$-consistent tuples.

\paragraph{We have $x_{1:n}\in\mathcal{E}_l$.}
By definition of $l$, for each $i$ there exists a program $p_i$ with $|p_i|\le l$ such that $U(p_i,x_{-i})=x_i$. Hence, when $\mathsf{Enum}_i(l)$ reaches $(p_i,x_{-i})$ in its dovetailing, it outputs exactly the tuple $x_{1:n}$. Therefore $x_{1:n}$ will eventually be marked $l$-consistent, so indeed $x_{1:n}\in\mathcal{E}_l$.

\paragraph{Step 2: A bounded-degree $n$-partite hypergraph.}
Define an $n$-partite $n$-uniform hypergraph $H_l=(V,E)$ as follows. For each $i\in[n]$, let the $i$-th part be
\[
V_i \;:=\; \bigl\{\, (i,t)\ :\ t\in(\{0,1\}^\ast)^{n-1}\,\bigr\},
\qquad
V \;:=\; \bigsqcup_{i=1}^n V_i.
\]
For each $z=(z_1,\dots,z_n)\in\mathcal{E}_l$, add a hyperedge
\[
e(z)\;:=\;\bigl\{\, (1,z_{-1}),\, (2,z_{-2}),\,\dots,\,(n,z_{-n}) \,\bigr\},
\]
i.e.\ the edge is incident to the projection $z_{-i}$ in part $i$ for every $i$.
Let $E:=\{e(z):z\in\mathcal{E}_l\}$.

\paragraph{Key degree bound.}
Fix any vertex $(i,t)\in V_i$. We claim that the number of edges incident to $(i,t)$ is at most $2^l$. Indeed, if an edge $e(z)$ is incident to $(i,t)$, then $z_{-i}=t$ and, since $z\in\mathcal{E}_l$, there exists a program $p$ with $|p|\le l$ such that $U(p,t)=z_i$. There are at most $2^l$ binary strings of length at most $l$, hence at most $2^l$ such programs $p$. Each fixed program $p$ on fixed input $t$ produces at most one output $z_i$. Therefore the number of \emph{distinct} possible values of $z_i$, and hence the number of distinct edges incident to $(i,t)$, is at most $2^l$. So the maximum degree of $H_l$ satisfies
\[
\Delta(H_l)\;\le\;2^l.
\]

\paragraph{Step 3: Greedy edge-coloring with a finite palette.}
We now define a computable coloring of edges of $H_l$ as they are enumerated. Let
\[
c \;:=\; n(2^l-1)+1.
\]
Process edges in the enumeration order induced by $\mathcal{E}_l$. When a new edge $e$ arrives, look at all previously colored edges that conflict with $e$, meaning they share at least one vertex with $e$. For each of the $n$ vertices of $e$, there are at most $2^l-1$ previously colored incident edges (by the degree bound above), hence the total number of conflicting colored edges is at most $n(2^l-1)=c-1$. Therefore there exists at least one color in $\{1,2,\dots,c\}$ not used by any conflicting edge. Assign to $e$ the smallest available such color.

Denote the resulting edge-color by $\chi_l(e)\in[c]$. By construction, $\chi_l$ satisfies the crucial uniqueness property: for every $i\in[n]$ and every $(i,t)\in V_i$, the colors $\chi_l(e)$ are pairwise distinct over all edges $e$ incident to $(i,t)$.

\paragraph{Step 4: Define the completion string $x_{n+1}$.}
Let $e^\star := e(x_{1:n})\in E$ be the edge corresponding to our given tuple.
Define
\[
x_{n+1} \;:=\; \bigl\langle\, l,\ \chi_l(e^\star)\,\bigr\rangle,
\]
where $\langle\cdot,\cdot\rangle$ is the fixed pairing function from \Cref{sec:kolmogorov-background}.

\paragraph{Size bound $K(x_{n+1})\stackrel{+}{\le} l$.}
Encoding the color $\chi_l(e^\star)\in[c]$ takes $\log c$ bits, and
\[
\log c \;=\; \log\bigl(n(2^l-1)+1\bigr) \;=\; l + O(\log n).
\]
Encoding the integer $l$ itself costs $K(l)=O(\log l)$ bits. Thus
\[
K(x_{n+1})
\;\stackrel{+}{\le}\;
\log c + K(l)
\;=\;
l + O(\log n)+O(\log l)
\;\stackrel{+}{\le}\; l,
\]
where the last step uses the fact that $O(\log n)+O(\log l)$ is within the allowed logarithmic slack.

\paragraph{Recover $x_{n+1}$ from $x_{1:n}$ (the constraint $K(x_{n+1}\mid x_{1:n})\stackrel{+}{=}0$).}
Given $x_{1:n}$, a program can (with $l$ hardcoded) simulate the enumeration of $\mathcal{E}_l$ and the greedy coloring procedure above until the tuple $x_{1:n}$ is declared $l$-consistent and its edge $e^\star$ is colored. At that moment the program outputs $\langle l,\chi_l(e^\star)\rangle=x_{n+1}$. Since the program only needs to encode $l$ (cost $K(l)=O(\log l)$), we get
\[
K(x_{n+1}\mid x_{1:n}) \;\le\; K(l)+O(1) \;=\; O(\log l)\;\stackrel{+}{=}0.
\]

\paragraph{Recover each missing component $x_i$ (the constraints $K(x_i\mid x_{-i},x_{n+1})\stackrel{+}{=}0$).}
Fix $i\in[n]$. Given $(x_{-i},x_{n+1})$, decode $x_{n+1}$ to obtain the pair $(l,\gamma)$ where $\gamma=\chi_l(e^\star)\in[c]$. Now simulate the same enumeration of $\mathcal{E}_l$ and greedy coloring, and watch for edges $e(z)$ that are incident to the vertex $(i,x_{-i})$, i.e.\ tuples $z$ satisfying $z_{-i}=x_{-i}$. Whenever such an edge is colored, check its color; if it equals $\gamma$, output its $i$-th component $z_i$ and halt.

This procedure halts and is correct because: (i) the target edge $e^\star=e(x_{1:n})$ is incident to $(i,x_{-i})$ and has color $\gamma$, so eventually it will be found; (ii) by the uniqueness property of $\chi_l$, among edges incident to $(i,x_{-i})$ there is \emph{at most one} edge of color $\gamma$, so the first match must be $e^\star$ and the output is $x_i$. The decoding program is constant-size (it only uses the universally fixed simulation code; $l$ and $\gamma$ are read from $x_{n+1}$), hence
\[
K(x_i\mid x_{-i},x_{n+1}) = O(1)\;\stackrel{+}{=}0.
\]

Combining the three verified constraints establishes \eqref{eq:one-var-feasibility}.
\end{proof} \begin{proof}[Proof of \Cref{prop:pivot-not-always}]
We argue by contradiction. Throughout, all $O(1)$ and $O(\log n)$ terms are uniform in $n$.

Fix $n$ large and let $p=p_n$ be the first prime in $[2^n,2^{n+1})$.
As in the statement, we represent elements of $\mathbb{F}_p$ as binary strings of length $\lceil \log p\rceil+O(1)$,
so that $n$ (hence $p$ and field arithmetic in $\mathbb{F}_p$) is computable from the encoding length by a constant-size program.
In particular $K(p)=O(\log n)$.

\paragraph{A line graph and a sampler bound.}
Let
\[
L=R=\mathbb{F}_p^\times\times \mathbb{F}_p,
\]
and define the bipartite graph $G_p=(L\sqcup R,E)$ by connecting $(a,c)\in L$ to $(b,d)\in R$
iff there exists $z\in\mathbb{F}_p^\times$ such that $b=za$ and $d=z+c$.
For each $(a,c)\in L$, its neighborhood is the $(p-1)$-point ``line''
\[
N(a,c)=L_{a,c}:=\{(za,\ z+c): z\in\mathbb{F}_p^\times\}\subseteq R.
\]
This graph is $(p-1)$-regular on both sides and $|L|=|R|=p(p-1)$. 

\begin{lemma}[Spectral mixing for $G_p$~\citep{HooryLinialWigderson2006}]
\label{lem:mixing}
For all $U\subseteq L$ and $V\subseteq R$,
\[
\Bigl|e(U,V)-\frac{|U||V|}{p}\Bigr|\ \le\ \sqrt p\ \sqrt{|U||V|}.
\]
\end{lemma}

A complete proof of \Cref{lem:mixing} is given in \Cref{app:mixing}.

\begin{lemma}[Sampler bound from mixing]
\label{lem:sampler-bound}
Fix $\alpha\in(0,1)$.
For any $S\subseteq R$ with $|S|\le (\alpha/2)p^2$, the set
\[
B(S,\alpha):=\bigl\{u\in L:\ |N(u)\cap S|\ge \alpha p\bigr\}
\]
satisfies $|B(S,\alpha)|\le \frac{4|S|}{\alpha^2\,p}$.
\end{lemma}

\begin{proof}
Let $U:=B(S,\alpha)$.
By definition, every $u\in U$ has at least $\alpha p$ neighbors in $S$, hence
\[
e(U,S)=\sum_{u\in U}|N(u)\cap S|\ \ge\ |U|\alpha p.
\]
On the other hand, by \Cref{lem:mixing},
\[
e(U,S)\ \le\ \frac{|U||S|}{p}+\sqrt p\sqrt{|U||S|}.
\]
Using $|S|\le (\alpha/2)p^2$ gives $\frac{|S|}{p}\le (\alpha/2)p$, so
\[
|U|\alpha p\ \le\ |U|\frac{\alpha}{2}p+\sqrt p\sqrt{|U||S|}
\quad\Rightarrow\quad
|U|\frac{\alpha}{2}p\ \le\ \sqrt p\sqrt{|U||S|}.
\]
Squaring and canceling $|U|$ yields $|U|\le \frac{4|S|}{\alpha^2 p}$.
\end{proof}

\paragraph{Ruling out all short programs on almost all lines.}
Take a large enough constant $k$ and set
\[
\widehat\alpha(n):=n^{-k}.
\]
Let $\mathcal{P}_n$ be the set of all programs of length at most $k\log n$.
Then $|\mathcal{P}_n|\le 2^{k\log n}=n^k$.

For a fixed program $Q\in\mathcal{P}_n$, view $Q$ as a partial function
\[
Q: R\to\{0,1\}^\ast,
\qquad (b,d)\mapsto U(Q,b,d).
\]
For each output string $w$, define its (halting) fiber
\[
S_w(Q):=\{(b,d)\in R:\ Q(b,d)\downarrow=w\}.
\]
These fibers form a partition of the halting set, so in particular
\begin{equation}
\sum_w |S_w(Q)|\ \le\ |R|\ =\ p(p-1)\ \le\ p^2.
\label{eq:sum-fibers}
\end{equation}
Define the ``bad line'' set for $Q$ by
\[
E_Q\ :=\ \Bigl\{(a,c)\in L:\ \exists w\ \text{s.t.}\ |S_w(Q)|<\tfrac{\widehat\alpha}{2}p^2\ \ \text{and}\ \ |S_w(Q)\cap L_{a,c}|\ge \widehat\alpha p\Bigr\}.
\]
In words: $(a,c)$ is bad for $Q$ if $Q$ has some output $w$ that is \emph{not} globally heavy
($|S_w(Q)|<(\widehat\alpha/2)p^2$) but is heavy \emph{along the line} $L_{a,c}$.

By \Cref{lem:sampler-bound}, for each fixed $w$ with $|S_w(Q)|<(\widehat\alpha/2)p^2$, the set of $(a,c)$ such that
$|S_w(Q)\cap L_{a,c}|\ge \widehat\alpha p$ has size at most $\frac{4|S_w(Q)|}{\widehat\alpha^2 p}$.
Taking the union over all such $w$ and using \eqref{eq:sum-fibers} gives
\begin{equation}
|E_Q|\ \le\ \sum_w \frac{4|S_w(Q)|}{\widehat\alpha^2 p}\ \le\ \frac{4}{\widehat\alpha^2 p}\cdot p^2\ =\ \frac{4p}{\widehat\alpha^2}.
\label{eq:EQ-bound}
\end{equation}
Now define the total exceptional set
\[
E\ :=\ \bigcup_{Q\in\mathcal{P}_n} E_Q.
\]
Using $|\mathcal{P}_n|\le n^k$ and \eqref{eq:EQ-bound},
\begin{equation}
|E|\ \le\ |\mathcal{P}_n|\cdot \frac{4p}{\widehat\alpha^2}
\ \le\ 4p\cdot n^k\cdot n^{2k}
\ =\ 4p\,n^{3k}
\ =\ o(p^2),
\label{eq:E-small}
\end{equation}
since $p\ge 2^n$ dominates any fixed polynomial in $n$.

\paragraph{Choose an incompressible good line $(A,C)$ and an independent $Z$.}
Because $|L|=p(p-1)=\Theta(p^2)$ and \eqref{eq:E-small} gives $|E|=o(p^2)$, we have $|L\setminus E|=\Theta(p^2)$.
By incompressibility, there exists $(A,C)\in L\setminus E$ such that
\[
K(A,C)\ \ge\ \log|L\setminus E|-O(1)\ =\ 2\log p -O(1)\ =\ 2n\pm O(1).
\]
In particular $K(A)\stackrel{+}{=}K(C)\stackrel{+}{=}n$, and also $A\in\mathbb{F}_p^\times$ by $(A,C)\in L$.

Next choose $Z\in\mathbb{F}_p^\times$ so that
\[
K(Z\mid A,C)\ \ge\ n-O(\log n).
\]
Such a $Z$ exists by counting: for any $t\ge 1$, at most $2^{n-t}$ values of $z\in\mathbb{F}_p^\times$
can satisfy $K(z\mid A,C)<n-t$, so taking $t=\Theta(\log n)$ ensures a choice with the stated bound.
Then $K(Z)\stackrel{+}{=}n$ and
\[
K(A,C,Z)\ \ge\ K(A,C)+K(Z\mid A,C)-O(\log n)\ \ge\ 3n-O(\log n).
\]

Now define
\[
B:=ZA\in\mathbb{F}_p^\times,\qquad D:=Z+C\in\mathbb{F}_p.
\]

\paragraph{The two given determination nodes hold.}
Since $A\neq 0$ and $Z\neq 0$, the map $(A,Z)\mapsto (A,ZA)=(A,B)$ is invertible with inverse
$Z=B/A$ and $A=B/Z$.
Similarly, $(Z,C)\mapsto (Z,Z+C)=(Z,D)$ is invertible with inverse $C=D-Z$ and $Z=D-C$.
Because arithmetic in $\mathbb{F}_p$ is computable by a constant-size program (given $p$),
each component of $(A,B,Z)$ is computable from the other two with $O(1)$ overhead, and likewise for $(Z,C,D)$.
Hence both triples are determinations, proving (i).

\paragraph{Assume for contradiction that a pivot exists.}
Suppose there exists a binary string $X$ such that both $(A,C,X)$ and $(B,D,X)$ are determinations.
In particular,
\[
K(X\mid A,C)\stackrel{+}{=}0,\qquad
K(X\mid B,D)\stackrel{+}{=}0,\qquad
K(B\mid D,X)\stackrel{+}{=}0.
\]

\paragraph{Lower bound: any such $X$ satisfies $K(X)\ge n-O(\log n)$.}
From $K(B\mid D,X)\stackrel{+}{=}0$ and the standard inequality
$K(U\mid V)\stackrel{+}{\le}K(W)+K(U\mid V,W)$, we get
\begin{equation}
K(B\mid D)\ \stackrel{+}{\le}\ K(X).
\label{eq:KB_D_le_KX_again}
\end{equation}
It remains to lower-bound $K(B\mid D)$ from the near-independence of $(A,C,Z)$.

\begin{lemma}
\label{lem:KB-given-D}
With $B=ZA$ and $D=Z+C$, the condition $K(A,C,Z)\ge 3n-O(\log n)$ implies
\[
K(B\mid D)\ \ge\ n-O(\log n).
\]
\end{lemma}

\begin{proof}
Consider the computable bijection
\[
(A,C,Z)\ \longleftrightarrow\ (B,D,Z),\qquad (B,D,Z)=(ZA,\ Z+C,\ Z),
\]
with inverse $(A,C,Z)=(B/Z,\ D-Z,\ Z)$.
Therefore $K(B,D,Z)=K(A,C,Z)\pm O(1)\ge 3n-O(\log n)$.
By the chain rule,
\[
K(B,D,Z)\ \stackrel{+}{\le}\ K(B,D)+K(Z\mid B,D),
\]
and trivially $K(Z\mid B,D)\le K(Z)\stackrel{+}{=}n$, hence $K(B,D)\ge 2n-O(\log n)$.
Finally,
\[
K(B\mid D)\ \stackrel{+}{\ge}\ K(B,D)-K(D)-O(\log n)\ \ge\ (2n-O(\log n))-(n+O(1))-O(\log n)=n-O(\log n),
\]
since $D\in\mathbb{F}_p$ has description length at most $\lceil\log p\rceil+O(1)=n+O(1)$.
\end{proof}

Combining \Cref{lem:KB-given-D} with \eqref{eq:KB_D_le_KX_again} yields
\begin{equation}
K(X)\ \ge\ n-O(\log n).
\label{eq:KX-lower-final}
\end{equation}

\paragraph{Upper bound: any such $X$ satisfies $K(X)=O(\log n)$.}
Since $K(X\mid B,D)\stackrel{+}{=}0$, there exists a witness program $Q$ of logarithm length
that outputs $X$ from $(B,D)$; for all sufficiently large $n$ we may assume (by enlarging $k$ if needed) that
\[
|Q|\ \le\ k\log n
\qquad\text{and}\qquad
U(Q,B,D)=X.
\]
Fix such a $Q\in\mathcal{P}_n$.

Define the fiber sets $S_w(Q)$ as above, and let $S_X(Q)$ denote the fiber for the particular output $X$.

\smallskip
\noindent\textbf{(a) Line-heaviness on the true line.}
Consider
\[
T\ :=\ \{z\in\mathbb{F}_p^\times:\ U(Q,zA,z+C)=X\}.
\]
Since $(B,D)=(ZA,Z+C)$, we have $Z\in T$, and moreover $|T|=|S_X(Q)\cap L_{A,C}|$.

Given $(A,C,X,Q)$, we can enumerate $T$ by dovetailing the computations $U(Q,zA,z+C)$ over all $z\in\mathbb{F}_p^\times$
and outputting those $z$ that halt with output $X$.
Therefore, specifying the index of $Z$ in this enumeration allows reconstructing $Z$ from $(A,C,X,Q)$, giving
\begin{equation}
K(Z\mid A,C,X,Q)\ \le\ \log|T|+O(1).
\label{eq:Z-index}
\end{equation}
On the other hand, by the chain rule,
\[
K(Z\mid A,C)\ \stackrel{+}{\le}\ K(Z\mid A,C,X,Q)+K(X,Q\mid A,C)+O(\log n).
\]
Since $K(X\mid A,C)\stackrel{+}{=}0$ and $|Q|\le k\log n$, we have $K(X,Q\mid A,C)=O(\log n)$, hence
\[
K(Z\mid A,C,X,Q)\ \ge\ K(Z\mid A,C)-O(\log n)\ \ge\ n-O(\log n).
\]
Combining with \eqref{eq:Z-index} yields $\log|T|\ge n-O(\log n)$, so for all sufficiently large $n$,
\[
|T|\ \ge\ \widehat\alpha(n)\,p,
\]
absorbing the polynomial slack into $\widehat\alpha(n)=n^{-k}$ by choosing $k$ large enough.
Equivalently,
\begin{equation}
|S_X(Q)\cap L_{A,C}|\ \ge\ \widehat\alpha\,p.
\label{eq:line-heavy}
\end{equation}

\smallskip
\noindent\textbf{(b) Good line $\Rightarrow$ global heaviness.}
Because we chose $(A,C)\in L\setminus E$, in particular $(A,C)\notin E_Q$.
By definition of $E_Q$, \eqref{eq:line-heavy} then forces
\begin{equation}
|S_X(Q)|\ \ge\ \tfrac{\widehat\alpha}{2}\,p^2.
\label{eq:global-heavy}
\end{equation}

\smallskip
\noindent\textbf{(c) Global heaviness $\Rightarrow$ $X$ is $O(\log n)$-describable.}
Define the set of globally heavy outputs of $Q$:
\[
\mathcal{W}_{\mathrm{heavy}}(Q)\ :=\ \Bigl\{w:\ |S_w(Q)|\ge \tfrac{\widehat\alpha}{2}\,p^2\Bigr\}.
\]
Using \eqref{eq:sum-fibers}, we have
\[
|\mathcal{W}_{\mathrm{heavy}}(Q)|\ \le\ \frac{2}{\widehat\alpha}\ =\ 2n^k.
\]
Moreover, $\mathcal{W}_{\mathrm{heavy}}(Q)$ is effectively enumerable given $(p,Q)$:
dovetail all computations $U(Q,b,d)$ over $(b,d)\in R$; maintain counts of discovered halting outputs;
and output $w$ the first time its count reaches $(\widehat\alpha/2)p^2$ (so each heavy $w$ is output exactly once).
Since $X\in\mathcal{W}_{\mathrm{heavy}}(Q)$ by \eqref{eq:global-heavy}, specifying the index of $X$ in this fixed enumeration
allows reconstructing $X$ from $(p,Q)$ and that index. Therefore
\[
K(X)\ \le\ K(p)+|Q|+\log|\mathcal{W}_{\mathrm{heavy}}(Q)|+O(1)\ =\ O(\log n).
\]
This contradicts the lower bound \eqref{eq:KX-lower-final} for all sufficiently large $n$.

\paragraph{Conclusion.}
Hence no such pivot $X$ exists, proving (ii).
\end{proof}
 \begin{proof}[Proof of \Cref{lem:det-difference-identity}]
Rewrite the claimed identity as
\[
\bigl[K(A)-K(B)\bigr]-\bigl[K(C\mid B)-K(C\mid A)\bigr]
\;=\;
\bigl[K(A)+K(C\mid A)\bigr]-\bigl[K(B)+K(C\mid B)\bigr].
\]
By the chain rule, $K(A)+K(C\mid A)\stackrel{+}{=}K(A,C)$ and $K(B)+K(C\mid B)\stackrel{+}{=}K(B,C)$.
Since $(A,B,C)$ is a determination, we have $K(B\mid A,C)\stackrel{+}{=}0$ and $K(A\mid B,C)\stackrel{+}{=}0$,
hence
\[
K(A,C)\ \stackrel{+}{=}\ K(A,B,C)\ \stackrel{+}{=}\ K(B,C).
\]
Therefore the difference is $\stackrel{+}{=}0$, as required.
\end{proof} 
\begin{proof}[Proof of \Cref{thm:ground-bounds-excess}]
First, for any strings $u,v$ we have the standard subadditivity bound
$K(u)\stackrel{+}{\le}K(v)+K(u\mid v)$ (describe $v$, then describe $u$ given $v$).
Applying this to $(u,v)=(A^{\star},\hat A)$ and $(u,v)=(B^{\star},\hat B)$ and adding gives
\[
K(A^{\star})+K(B^{\star})\ \stackrel{+}{\le}\ K(\hat A)+K(\hat B)\ +\ K(A^{\star}\mid \hat A)+K(B^{\star}\mid \hat B).
\]
The last two terms are exactly the optimal objective value of \eqref{eq:grounded-split-opt} (up to $\stackrel{+}{=}$).
Now observe that $(A,B)=(I,\epsilon)$ is feasible for \eqref{eq:grounded-split-opt}, hence optimality implies
\[
K(A^{\star}\mid \hat A)+K(B^{\star}\mid \hat B)\ \stackrel{+}{\le}\ K(I\mid \hat A)+K(\epsilon\mid \hat B)
\ \stackrel{+}{\le}\ K(I),
\]
where $K(\epsilon\mid \hat B)=O(1)$ and conditioning cannot increase complexity by more than $O(1)$.
Combining the two displays yields
$K(A^{\star})+K(B^{\star})\stackrel{+}{\le}K(\hat A)+K(\hat B)+K(I)$, which rearranges to the claim.
\end{proof} \begin{proof}[Proof of \Cref{fig:delta-relations}]
All relations are up to $\stackrel{+}{=}$ / $\stackrel{+}{\le}$.

\smallskip
\noindent\textbf{(0) No assumption: $\Delta^{(7)}\stackrel{+}{=}2\Delta^{(1)}-\Delta^{(5)}$.}
\begin{align*}
\Delta^{(7)}
&=K(I\mid B)+K(B\mid I)-K(P_1\mid P_2)-K(P_2\mid P_1)+K(B)\\
&\stackrel{+}{=}[K(I,B)-K(B)]+[K(I,B)-K(I)]-[K(P_1,P_2)-K(P_2)]-[K(P_1,P_2)-K(P_1)]+K(B)\\
&\stackrel{+}{=}K(P_1)+K(P_2)-2K(P_1,P_2)-K(I)+2K(I,B).
\end{align*}
Also
\begin{align*}
2\Delta^{(1)}-\Delta^{(5)}
&\stackrel{+}{=}2\big[I(P_1;P_2)+K(B\mid I)\big]-\big[K(P_1)+K(P_2)-K(I)\big]\\
&\stackrel{+}{=}2\big[K(P_1)+K(P_2)-K(P_1,P_2)\big]+2\big[K(I,B)-K(I)\big]-K(P_1)-K(P_2)+K(I)\\
&\stackrel{+}{=}K(P_1)+K(P_2)-2K(P_1,P_2)-K(I)+2K(I,B).
\end{align*}

\smallskip
\noindent\textbf{(A1) Assume $K(B\mid P_i)\stackrel{+}{=}0$.}

\emph{(A1$\Rightarrow$ $\Delta^{(1)}\stackrel{+}{=}\Delta^{(2)}$).}
First, since $K(B\mid P_i)\stackrel{+}{=}0$, we have
\[
K(P_i,B)\stackrel{+}{=}K(P_i)
\qquad\text{and}\qquad
K(P_i,B)\stackrel{+}{=}K(B)+K(P_i\mid B),
\]
hence
\[
K(P_i)\stackrel{+}{=}K(B)+K(P_i\mid B),\qquad i=1,2.
\]
Moreover, $K(B\mid P_1,P_2)\stackrel{+}{=}0$ (since $B$ is computable from $P_1$ alone), so
\[
K(P_1,P_2,B)\stackrel{+}{=}K(P_1,P_2)
\qquad\text{and}\qquad
K(P_1,P_2,B)\stackrel{+}{=}K(B)+K(P_1,P_2\mid B),
\]
hence
\[
K(P_1,P_2)\stackrel{+}{=}K(B)+K(P_1,P_2\mid B).
\]
Therefore
\begin{align*}
I(P_1;P_2)
&:=K(P_1)+K(P_2)-K(P_1,P_2)\\
&\stackrel{+}{=}\big[K(B)+K(P_1\mid B)\big]+\big[K(B)+K(P_2\mid B)\big]-\big[K(B)+K(P_1,P_2\mid B)\big]\\
&=K(B)+\big[K(P_1\mid B)+K(P_2\mid B)-K(P_1,P_2\mid B)\big]\\
&=K(B)+I(P_1;P_2\mid B).
\end{align*}
Hence
\[
\Delta^{(1)}=I(P_1;P_2)+K(B\mid I)
\stackrel{+}{=}K(B)+I(P_1;P_2\mid B)+K(B\mid I)=\Delta^{(2)}.
\]

\emph{(A1$\Rightarrow$ $\Delta^{(5)}\stackrel{+}{=}\Delta^{(4)}$).}
For each $i$,
\[
K(P_i,B)\stackrel{+}{\le}K(P_i)+K(B\mid P_i)\stackrel{+}{=}K(P_i),
\qquad
K(P_i)\stackrel{+}{\le}K(P_i,B),
\]
so $K(P_i,B)\stackrel{+}{=}K(P_i)$ and
\[
\Delta^{(4)}=K(P_1,B)+K(P_2,B)-K(I)\stackrel{+}{=}K(P_1)+K(P_2)-K(I)=\Delta^{(5)}.
\]

\emph{(A1$\Rightarrow$ $\Delta^{(6)}\stackrel{+}{\ge}\Delta^{(2)}$).}
We compute
\[
\Delta^{(6)}-\Delta^{(2)}\stackrel{+}{=}I(P_1;P_2\mid I)-K(B\mid I),
\]
so it suffices to show $I(P_1;P_2\mid I)\stackrel{+}{\ge}K(B\mid I)$.

Under A1 we have $K(B\mid P_i,I)\stackrel{+}{=}0$ for $i=1,2$.
Thus, for each $i$,
\[
K(P_i\mid I)\stackrel{+}{=}K(B,P_i\mid I)\stackrel{+}{=}K(B\mid I)+K(P_i\mid B,I).
\]
Also, by an explicit description (first describe $B$ from $I$, then describe $P_1$ and $P_2$ given $(B,I)$),
\[
K(P_1,P_2\mid I)\stackrel{+}{\le}K(B\mid I)+K(P_1\mid B,I)+K(P_2\mid B,I).
\]
Therefore
\begin{align*}
I(P_1;P_2\mid I)
&:=K(P_1\mid I)+K(P_2\mid I)-K(P_1,P_2\mid I)\\
&\stackrel{+}{\ge}\big[K(B\mid I)+K(P_1\mid B,I)\big]+\big[K(B\mid I)+K(P_2\mid B,I)\big]\\
&\quad-\big[K(B\mid I)+K(P_1\mid B,I)+K(P_2\mid B,I)\big]\\
&=K(B\mid I).
\end{align*}
Hence $\Delta^{(6)}-\Delta^{(2)}\stackrel{+}{\ge}0$, i.e. $\Delta^{(6)}\stackrel{+}{\ge}\Delta^{(2)}$.

\smallskip
\noindent\textbf{(A2) Assume $K(P_i\mid I,B)\stackrel{+}{=}0$.}

\emph{(A2$\Rightarrow$ $\Delta^{(2)}\stackrel{+}{=}\Delta^{(3)}$).}
A2 implies
\[
K(P_1,P_2\mid I,B)\stackrel{+}{\le}K(P_1\mid I,B)+K(P_2\mid I,B)\stackrel{+}{=}0,
\]
hence $K(P_1,P_2\mid I,B)\stackrel{+}{=}0$ and
\[
I\big((P_1,P_2);I\mid B\big)\stackrel{+}{=}K(P_1,P_2\mid B)-K(P_1,P_2\mid I,B)\stackrel{+}{=}K(P_1,P_2\mid B).
\]
Therefore
\begin{align*}
\Delta^{(3)}
&=K(B)+K(P_1\mid B)+K(P_2\mid B)+K(B\mid I)-I((P_1,P_2);I\mid B)\\
&\stackrel{+}{=}K(B)+K(P_1\mid B)+K(P_2\mid B)+K(B\mid I)-K(P_1,P_2\mid B)\\
&\stackrel{+}{=}K(B)+I(P_1;P_2\mid B)+K(B\mid I)=\Delta^{(2)}.
\end{align*}

\emph{(A2$\Rightarrow$ $\Delta^{(6)}\stackrel{+}{\le}\Delta^{(2)}$).}
Again
\[
\Delta^{(6)}-\Delta^{(2)}\stackrel{+}{=}I(P_1;P_2\mid I)-K(B\mid I).
\]
But
\[
I(P_1;P_2\mid I)\stackrel{+}{\le}K(P_1\mid I)
\stackrel{+}{\le}K(P_1,B\mid I)\stackrel{+}{=}K(B\mid I)+K(P_1\mid I,B)\stackrel{+}{=}K(B\mid I),
\]
so $\Delta^{(6)}\stackrel{+}{\le}\Delta^{(2)}$.

\smallskip
\noindent\textbf{(A3) Assume $K(I\mid P_1,P_2)\stackrel{+}{=}0$.}

\emph{(A3$\Rightarrow$ $\Delta^{(3)}\stackrel{+}{=}\Delta^{(4)}$).}
Use $K(P_i,B)\stackrel{+}{=}K(B)+K(P_i\mid B)$ to rewrite
\[
\Delta^{(4)}\stackrel{+}{=}2K(B)+K(P_1\mid B)+K(P_2\mid B)-K(I).
\]
Hence
\begin{align*}
\Delta^{(3)}-\Delta^{(4)}
&\stackrel{+}{=}\big[K(B)+K(P_1\mid B)+K(P_2\mid B)+K(B\mid I)-I((P_1,P_2);I\mid B)\big]\\
&\quad-\big[2K(B)+K(P_1\mid B)+K(P_2\mid B)-K(I)\big]\\
&\stackrel{+}{=}K(B\mid I)-K(B)+K(I)-I((P_1,P_2);I\mid B)\\
&\stackrel{+}{=}-I(B;I)+K(I)-\big[K(I\mid B)-K(I\mid B,P_1,P_2)\big]\\
&\stackrel{+}{=}K(I\mid B,P_1,P_2).
\end{align*}
By A3, $K(I\mid B,P_1,P_2)\stackrel{+}{\le}K(I\mid P_1,P_2)\stackrel{+}{=}0$, hence
$\Delta^{(3)}\stackrel{+}{=}\Delta^{(4)}$.

\end{proof}
 \subsection{Proof of the Mixing Lemma}
\label{app:mixing}

In this appendix we prove Lemma~\ref{lem:mixing} in a self-contained way.

\subsubsection{Additive and Multiplicative Characters}
Fix a prime $p$.
For $t\in\mathbb{F}_p$, let the additive character $\psi_t:\mathbb{F}_p\to\mathbb{C}$ be
\[
\psi_t(x):=\exp(2\pi i\,t x/p).
\]
For $s\in\mathbb{F}_p^\times$, let $\chi$ denote a multiplicative character $\chi:\mathbb{F}_p^\times\to\mathbb{C}$ (a group homomorphism);
extend it to $\mathbb{F}_p$ by setting $\chi(0):=0$.
Write $\chi=\chi_0$ for the trivial character.

\subsubsection{Gauss Sum Bound}
\begin{lemma}[Gauss sum bound]
\label{lem:gauss}
Let $\chi$ be a nontrivial multiplicative character of $\mathbb{F}_p^\times$ and let $t\in\mathbb{F}_p^\times$.
Then
\[
\Bigl|\sum_{z\in\mathbb{F}_p^\times}\chi(z)\psi_t(z)\Bigr|\ =\ \sqrt p.
\]
In particular, it is $\le \sqrt p$~\citep{IwaniecKowalski2004}.
\end{lemma}

\begin{proof}
By the change of variables $z\mapsto tz$ (using that $t\neq 0$), it suffices to treat $t=1$.
Define the Gauss sum
\[
\tau(\chi):=\sum_{z\in\mathbb{F}_p}\chi(z)\psi_1(z)
\ =\ \sum_{z\in\mathbb{F}_p^\times}\chi(z)\psi_1(z),
\]
since $\chi(0)=0$.
We compute $\tau(\chi)\overline{\tau(\chi)}=\tau(\chi)\tau(\overline{\chi})$:
\begin{align*}
\tau(\chi)\tau(\overline{\chi})
&=\sum_{x\in\mathbb{F}_p}\sum_{y\in\mathbb{F}_p}\chi(x)\overline{\chi(y)}\psi_1(x-y)\\
&=\sum_{u\in\mathbb{F}_p}\psi_1(u)\sum_{y\in\mathbb{F}_p}\chi(y+u)\overline{\chi(y)}.
\end{align*}
For $u=0$, the inner sum is $\sum_{y\in\mathbb{F}_p}|\chi(y)|^2=p-1$ (since $\chi(0)=0$ and $|\chi(y)|=1$ for $y\neq 0$).
For $u\neq 0$, the inner sum is independent of $u$:
\[
\sum_{y\in\mathbb{F}_p}\chi(y+u)\overline{\chi(y)}
=\sum_{y\in\mathbb{F}_p^\times}\chi(y+u)\overline{\chi(y)}
=\sum_{y\in\mathbb{F}_p^\times}\chi\!\Bigl(1+\frac{u}{y}\Bigr).
\]
As $y$ ranges over $\mathbb{F}_p^\times$, $t:=u/y$ ranges over $\mathbb{F}_p^\times$, so this equals
\[
\sum_{t\in\mathbb{F}_p^\times}\chi(1+t)
=\sum_{s\in\mathbb{F}_p\setminus\{1\}}\chi(s)
=\sum_{s\in\mathbb{F}_p^\times}\chi(s)-\chi(1)
=0-1
=-1,
\]
using $\sum_{s\in\mathbb{F}_p^\times}\chi(s)=0$ for nontrivial $\chi$ and $\chi(1)=1$.
Therefore
\[
\tau(\chi)\tau(\overline{\chi})
=(p-1)\psi_1(0)+\sum_{u\in\mathbb{F}_p^\times}\psi_1(u)\cdot(-1)
=(p-1)-\Bigl(\sum_{u\in\mathbb{F}_p^\times}\psi_1(u)\Bigr).
\]
Finally $\sum_{u\in\mathbb{F}_p}\psi_1(u)=0$, hence $\sum_{u\in\mathbb{F}_p^\times}\psi_1(u)=-1$, giving
\[
\tau(\chi)\tau(\overline{\chi})=(p-1)-(-1)=p.
\]
Thus $|\tau(\chi)|=\sqrt p$, proving the lemma.
\end{proof}

\subsubsection{Diagonalization of the Adjacency Operator and Expander Mixing}
Let $L=R=\mathbb{F}_p^\times\times\mathbb{F}_p$ and define the bipartite graph $G_p$ as in the main proof.
Let $T:\ell^2(L)\to\ell^2(R)$ be the (unweighted) adjacency operator
\[
(Tf)(b,d):=\sum_{\substack{(a,c)\in L\\(a,c)\sim(b,d)}} f(a,c)
=\sum_{z\in\mathbb{F}_p^\times} f\!\left(\frac{b}{z},\,d-z\right),
\]
and $T^\top:\ell^2(R)\to\ell^2(L)$ be its transpose
\[
(T^\top g)(a,c):=\sum_{z\in\mathbb{F}_p^\times} g(za,z+c).
\]

Consider the Fourier basis on $L$ given by
\[
f_{\chi,t}(a,c):=\chi(a)\psi_t(c),
\qquad \chi \text{ a multiplicative character on }\mathbb{F}_p^\times,\ \ t\in\mathbb{F}_p.
\]
A direct computation gives
\begin{align*}
(T^\top f_{\chi,t})(a,c)
&=\sum_{z\in\mathbb{F}_p^\times} f_{\chi,t}(za,z+c)\\
&=\sum_{z\in\mathbb{F}_p^\times}\chi(za)\psi_t(z+c)\\
&=\chi(a)\psi_t(c)\sum_{z\in\mathbb{F}_p^\times}\chi(z)\psi_t(z).
\end{align*}
Thus each $f_{\chi,t}$ is an eigenfunction of $T^\top$ with eigenvalue
\[
\lambda_{\chi,t}:=\sum_{z\in\mathbb{F}_p^\times}\chi(z)\psi_t(z).
\]
If $\chi=\chi_0$ and $t=0$, then $\lambda_{\chi,t}=p-1$ (the trivial eigenvalue).
Otherwise, either $\chi$ is nontrivial or $t\neq 0$; in all nontrivial cases one has
$|\lambda_{\chi,t}|\le \sqrt p$ by Lemma~\ref{lem:gauss} (or by the orthogonality of $\psi_t$ when $\chi=\chi_0$ and $t\neq 0$).

Therefore, the second largest singular value of the bipartite adjacency operator is at most $\sqrt p$.
The standard expander mixing lemma for $(p-1)$-regular bipartite graphs with $|L|=|R|=p(p-1)$ implies that
for all $U\subseteq L$ and $V\subseteq R$,
\[
\Bigl|e(U,V)-\frac{(p-1)|U||V|}{|R|}\Bigr|
\ \le\
\sqrt p\,\sqrt{|U||V|}
\qquad\Longrightarrow\qquad
\Bigl|e(U,V)-\frac{|U||V|}{p}\Bigr|
\ \le\
\sqrt p\,\sqrt{|U||V|},
\]
which is exactly Lemma~\ref{lem:mixing}.
~\citep{HooryLinialWigderson2006}.
  \end{document}